\documentclass[twoside, 11pt]{article}
\pdfoutput=1

\usepackage{jmlr2e}

\usepackage{algorithm}
\usepackage[noend]{algorithmic}
\usepackage{amsfonts}
\usepackage{amsmath}
\usepackage{color}
\usepackage{mathrsfs}
\usepackage{enumitem}
\usepackage{bm}
\usepackage{hyperref}
\usepackage{graphicx}
\graphicspath{{./figures/}}
\usepackage{amssymb}
\usepackage[suppress]{color-edits}
\addauthor{ll}{blue}

\newcommand{\event}{\Omega}

\definecolor{innerboxcolor}{rgb}{.9,.95,1}
\definecolor{outerlinecolor}{rgb}{.6,0,.2}

\renewcommand{\P}{\mathbb{P}}
\renewcommand{\Pr}{\mathbb{P}}
\newcommand{\history}{h}

\newcommand{\errorpot}[3]{\Lambda^{(#1)}_{#2}[#3]}

\newcommand{\psucc}[1]{\succ_{p_#1}}
\newcommand{\asucc}[1]{\succ_{a_#1}}

\newcommand{\match}{m}
\newcommand{\mstar}{M^*}

\newcommand{\stayprob}{\lambda}
\newcommand{\mstable}{M^{*}}
\newcommand{\agentset}{\mathcal{N}}
\newcommand{\armset}{\mathcal{K}}
\newcommand{\potset}[1]{S^{(#1)}}
\newcommand{\pull}[1]{\bar{A}^{(#1)} }
\newcommand{\attempt}[1]{A^{(#1)} }

\newcommand{\defeq}{\mathrel{\mathop:}=}
\newcommand{\mingap}{\Delta}
\newcommand{\gap}[2]{\Delta^{(#1)}_#2}

\newcommand{\delaydraw}[1]{D^{(#1)}}

\newcommand{\argmax}{\mathop{\rm argmax}}
\newcommand{\Ind}[1]{\mathbf{1}\left\{#1\right\}}

\newcommand{\E}{\mathbb{E}}

\newcommand{\player}[1]{p_{#1}}
\newcommand{\arm}[1]{a_{#1}}
 
\newcommand{\players}{\mathcal{N}}
\newcommand{\arms}{\mathcal{K}}

\newcommand{\meanreward}[2]{\mu^{(#1)}_{#2}} % agent, arm

\newcommand{\hatmeanreward}[3]{\widehat{\mu}^{(#1)}_{#2}(#3)} % agent, arm, time

\newcommand{\reward}[3]{X^{(#1)}_{#2}({#3})}

\newcommand{\actions}[1]{m_{{#1}}}

\newcommand{\optmatches}{\overline{m}}
\newcommand{\pessmatches}{\underline{m}}

\newcommand{\pessmatch}[1]{\underline{m}(#1)}

\newcommand{\action}[2]{m_{{#1}}({#2})}

\newcommand{\numarms}{L}
\newcommand{\numagents}{N}
\newcommand{\horizon}{T}

\newcommand{\pessgap}[2]{\underline{\Delta}_{#1, #2}}
\newcommand{\optregret}[1]{\overline{R}_{#1}}
\newcommand{\pessregret}[1]{\underline{R}_{#1}}
\newcommand{\regret}[1]{R_{#1}}

\newcommand{\numattempts}[3]{T^{(#1)}_{ #2}({#3})}
\newcommand{\numpulls}[3]{\bar{T}^{(#1)}_{ #2}({#3})}
\newcommand{\ucb}[3]{\text{UCB}^{({#1})}_{ {#2}} ({#3})} % agent, arm, time

\newtheorem{assumption}{Assumption}

\newcommand{\EE}{\mathbb{E}}
\newcommand{\Ocal}{\mathcal{O}}  
 
 \usepackage[normalem]{ulem}

 \jmlrheading{1}{2021}{1-48}{4/00}{10/00}{meila00a}{Lydia T.~Liu, Feng Ruan, Horia Mania and Michael I. Jordan}
 
 \ShortHeadings{Bandit Learning in Decentralized Matching Markets}{Liu, Ruan, Mania and Jordan}
 \firstpageno{1}

\begin{document}
 
\title{\textbf{Bandit Learning in Decentralized Matching Markets}}

\author{\name Lydia T.~Liu \email lydiatliu@cs.berkeley.edu
	\AND
	\name Feng Ruan \email fengruan@cs.berkeley.edu 
		\AND
		\name Horia Mania \email hmania@cs.berkeley.edu 
			\AND
	\name Michael I.\ Jordan \email jordan@cs.berkeley.edu~~\\
	\addr Department of Electrical Engineering and Computer Sciences and Department of Statistics\\
	University of California\\
	Berkeley, CA 94720-1776, USA}

\editor{}

\maketitle 

%\tableofcontents

\begin{abstract}
	We study two-sided matching markets in which one side of the market (the players) does not have a priori knowledge about its preferences for the other side (the arms) and is required to learn its preferences from experience. Also, we assume the players have no direct means of communication. This model extends the standard stochastic multi-armed bandit framework to a decentralized multiple player setting with competition. We introduce a new algorithm for this setting that, over a time horizon $T$, attains $\Ocal(\log(\horizon))$ stable regret when preferences of the arms over players are shared, and $\Ocal(\log(\horizon)^2)$ regret when there are no assumptions on the preferences on either side. Moreover, in the setting where a single player may deviate, we show that the algorithm is incentive compatible whenever the arms' preferences are shared, but not necessarily so when preferences are fully general.
\end{abstract}


\section{Introduction}

A fundamental question at the intersection of learning theory and game theory is as follows: \emph{how should individually rational agents act when they have to learn about the consequences of their actions in the same uncertain environment?} \lledit{While there has been a long line of work on learning in games \citep{fudenberg1998theory,hu1998multiagent,littman1994markov}, recent developments in statistical learning theory and online learning have opened the door to a new line of work that aims to quantify precisely the amount of data players require to achieve good performance in games with stochasticity. The problems studied are motivated by a broad range of modern applications, from modeling competition among firms \citep{MansourSW18,Aridor2019competing} to implementing protocols for wireless networks \citep{Liu10distributed, pmlr-v49-cesa-bianchi16, Shahrampour17}.} A particularly salient application is the online marketplace\footnote{Examples include online labor markets (Upwork, TaskRabbit, Handy), online crowdsourcing platforms (Amazon Mechanical Turk), online dating services (Match.com) and peer-to-peer sharing platforms (Airbnb).}, where two sides of a market need to be matched and market participants have uncertainty about their preferences, leading to a concomitant need for exploration and \lledit{statistical learning.}

The multi-armed bandit is a core learning problem that models decision-making under uncertainty: a player is faced with a choice among $K$ actions---``arms''---each of which is associated with a reward distribution, and the goal is to learn which action has the highest reward, doing so as quickly as possible so as to be able to reap rewards even while the learning process is underway. Even in the more complex setting involving multiple players participating in a two-sided matching market, the bandit problem can be extended to model how players simultaneously learn and acquire information about their preferences, while satisfying economic constraints imposed by the need to realize a matching. %To introduce a game-theoretic aspect into the bandit problem, it is natural to place the problem into the context of a two-way matching market, where the choices faced by the players are identified with the entities on the other side of the market, and where the need to realize a matching imposes economic constraints and incentives.  
\lledit{Such a blend of bandit learning with two-sided matching markets was  introduced by \citet{Das2005two}, who formulated a problem in which the players and the arms form the two sides of the market, and each side has preferences over the other side. \citet{Das2005two} explored possible algorithms via numerical simulations. \citet{liu20competing} studied a refinement of this problem setting and proposed the first algorithm with theoretical guarantees. In contradistinction to the classical formulation of matching markets, the preferences of the players are assumed to be unknown a priori and must be learned from the rewards that are received when arms are pulled successfully. 
Compared to prior work studying multi-player bandits, the problem formulation we consider introduces an aspect of scarcity and competition---when multiple players attempt to pull the same arm, there is a conflict, and only the player that is most preferred by that arm receives a reward.}

\citet{liu20competing} focused on a \emph{centralized} setting in which the players are able to communicate with a central platform that computes matchings for the entire market.  They defined a notion of regret called \emph{stable regret}, which is the average reward a player obtains less the rewards achieved under a stable matching with respect to the true preferences of the market. It was shown in this setting that an algorithm that combines the upper confidence bound principle from the bandit literature \citep{Lai85asymp} with the Gale-Shapley algorithm from the matching market literature \citep{gale62college} can achieve low stable regret.

\lledit{
While \citet{liu20competing} discussed a decentralized version of the problem, where the actions of the players cannot be coordinated by a central platform, and studied a simple explore-then-commit algorithm for this setting, finding a viable algorithm for the decentralized case was left as an open problem. }
The decentralized setting is arguably a more useful formulation in practice.  Indeed, most online marketplaces are decentralized, that is, there is no central clearinghouse and players are unable to coordinate their actions with each other directly. However, players may observe limited information about past matchings, such as their own conflicts%\footnote{E.g., when one user loses the competition to book a vacation rental on Airbnb to another.}
. 

New theoretical challenges arise in the decentralized setting, in both the design and the analysis of algorithms.  Given that players may use past matchings to inform their current play (e.g., to avoid conflicts), a player who has statistical uncertainty about their preferences over arms may impose externalities on other players not only at the current time step but also into the future. In essence, the decentralized formulation more fully exposes the challenges of the economic and learning aspects of the problem.

We propose a solution for the decentralized version of the two-sided matching bandit problem.  Our primary contribution is a new multiplayer bandit algorithm, \emph{Decentralized Conflict-Avoiding Upper Confidence Bound} (CA-UCB), that is guaranteed to yield for all players a stable regret that grows polylogarithmically with the number of rounds of interaction between players and arms, also known as the time horizon, $\horizon$. \lledit{In particular, to prove this regret guarantee we roughly showed that the market converges to a stable matching at a polylogarithmic rate.} When the arms have the same preferences over players we offer a better guarantee. In this case we prove that the stable regret grows at most logarithmically with the time horizon. Informally, we can state our results as follows.

\begin{theorem}[Informal main results]
	Suppose we have a market with $\numagents$ players and $\numarms$ arms, with arbitrary preferences, and let $\Delta$ be the minimum absolute gap between the mean rewards of different arms. Then, if all players run the CA-UCB algorithm  for $T$ steps, \lledit{the probability that the market is unstable at time $T$ is $\Ocal(\log(T)^2/T)$ (see Theorem~\ref{thm:regret_general_pref}). Moreover, the players' stable regret satisfies
	\begin{align*}
		R(\horizon) = \Ocal \left(\rho^{\numagents^4} \frac{\log(\horizon)^2}{\Delta^2}\right), \text{ for some } \rho > 1.
		\tag{Corollary~\ref{cor:general_regret}} %\tag{Theorem~\ref{thm:regret_general_pref}}
	\end{align*}}
When the arms have the same preferences over players, the players' stable regret satisfies 
	\begin{align*}
		R(\horizon) = \Ocal \left(\numagents^2 \numarms \frac{\log(\horizon)}{\Delta^2}\right). \tag{Theorem~\ref{thm:global-players}}
	\end{align*}
	Moreover, if $\numagents - 1$ players implement the CA-UCB algorithm, the remaining player cannot significantly improve their regret by running a different algorithm (Proposition~\ref{prop:incentive-global}).
\end{theorem}

%for two main classes of preference structures for the two-sided market: (i) where the arms have the same preferences, but the players do not (Section~\ref{sec:global-players}); (ii) where both the arms and players can have \emph{arbitrary} preferences and the stable matching may not be unique (Section~\ref{sec:general-prefs}). In the former case, the stable regret is found to be logarithmic in the horizon, matching the optimal rate found in the single-player bandit problem.  In this case, the stable regret scales as polynomial in the size of the market. In the second case, where preferences can be arbitrary, we are able to establish an upper bound on the regret that is polylogarithmic in the horizon, but with a dependence on the market size may be exponential.

%\todo{other things to add: what the conflict avoidance does.}

The %decentralized conflict-avoiding upper confidence bound
CA-UCB algorithm is simple and does not require communication between players.
There are two features of this algorithm that enable players to avoid conflicts. Firstly, when implementing this algorithm a player observes the actions of other players in the previous round and avoids attempting an arm if that arm was previously pulled by a better player for it. Secondly, players randomly decide whether to choose the same arm as at the previous time step or to make a new decision. When players implement our method conflicts can still occur, but our analysis shows that the expected number of conflicts would be small. 

%$A key feature that enables it to achieve low stable regret, even under no assumptions on the preference structure, is its use of randomization as an implicit conflict avoidance mechanism among players. The techniques developed in this work to combine randomized conflict avoidance with the upper confidence bound principle to reach a stable matching may be of independent interest. %The techniques developed in this work to harness the benefits of randomness for efficient convergence to a stable matching may be of independent interest. 

The rest of the paper is organized as follows: In Section~\ref{sec:problem}, we review the matching bandits problem, following the presentation in \citet{liu20competing}, and fully specify the decentralized setting that is our focus. In Section~\ref{sec:alg}, we motivate and introduce the algorithm that is the subject of our regret analyses in Sections~\ref{sec:global-players} and~\ref{sec:general-prefs}. In Section~\ref{sec:ic}, we discuss the incentive compatibility of this algorithm, showing one positive and one negative result. \lledit{Our theoretical guarantee on the performance of CA-UCB exhibits an exponential dependence on the size of the market. In Section~\ref{sec:expt} we show empirically that this dependence is an artifact of our analysis; CA-UCB performs much better in practice than these results suggest. }
In Section~\ref{sec:relwork}, we survey the related literatures, and in Section~\ref{sec:dis}, we present a thorough discussion of our results, %including a detailed comparison to related work by \citet{Sankararaman20dominate}, 
as well as avenues for future work.

\section{Problem Setting}\label{sec:problem}

We consider a multiplayer multi-armed bandit problem with $\numagents$ players and $\numarms$ stochastic arms, with $\numagents \leq \numarms$. We denote the set of players by $\agentset = \{\player{1}, \player{2}, \ldots, \player{\numagents}\}$ 
and the set of arms by $\armset = \{ \arm{1}, \arm{2}, \ldots, \arm{\numarms}\}$. 
At time step $t$, each player $\player{i}$ attempts to pull an arm $\action{t}{i} \in\arms$. 

When multiple players attempt to pull the same arm, only one player will successfully pull the arm, according to the 
arm's preferences via a mechanism we detail shortly. Then, if player $\player{i}$ successfully pulls arm $\action{t}{i}$ at time $t$, they are said to be \emph{matched} to $\action{t}{i}$ at time~$t$ and they receive a stochastic reward, $\reward{i}{\actions{t}}{t}$,  sampled from a $1$-sub-Gaussian distribution with mean~$\meanreward{i}{\action{t}{i}}> 0$.

For each player $\player{i}$ we assume $\meanreward{i}{j} \neq \meanreward{i}{j^\prime}$ for all distinct arms, $\arm{j}$ and $\arm{j^\prime}$. If $\meanreward{i}{j} > \meanreward{i}{j^\prime}$, we say that player $\player{i}$ \emph{truly prefers} $\arm{j}$ to $\arm{j'}$, and denote this as $\arm{j} \psucc{i} \arm{j'}$.

Each arm $\arm{j}$ has a fixed, known, and strict preference ordering over all the players, $\asucc{j}$. In other words, $\player{i} \asucc{j} \player{i^{\prime}}$ indicates that arm $\arm{j}$ prefers player $\player{i}$ to player $\player{i^{\prime}}$.
If two or more players attempt to pull the same arm $\arm{j}$, there is a \emph{conflict} and only the most preferred player successfully pulls the arm to receive a reward; the other player(s) $\player{i^\prime}$ is said to be \emph{unmatched} and does not receive any reward, that is, $\reward{i^\prime}{\actions{t}}{t} = 0$.

A \emph{stable matching} \citep{gale62college} of players and arms is one where no pair of player and arm would prefer to be matched with each other over their respective matches. Given the full preferences of the arms and players, arm $\arm{j}$ is called a \emph{achievable match} of player $\player{i}$ if there exists a stable matching according to those preferences such that $\arm{j}$ and~$\player{i}$ are matched. 
We say $\arm{j}$ is the \emph{optimal match} of player $\player{i}$ if it is the most preferred achievable match. Similarly, we say $\arm{j}$ is the \emph{pessimal match} of player $\player{i}$ if it is the least preferred achievable match. 
We denote by $\optmatches$ and $\pessmatches$ the functions from $\players$ to $\arms$ that define the optimal and pessimal matches of a player according to the true preferences of the players and arms.

In the decentralized matching setting, a notion of \emph{stable regret}, as introduced in \citet{liu20competing}, is useful for analyzing the performance of learning algorithms. We consider a player's \emph{player-pessimal stable regret}, where the baseline for comparison is the mean reward of the arm that is the player's pessimal match.\footnote{We can define analogously the \emph{player-optimal stable regret} corresponding to the player's optimal match, denoted $\optregret{i}(\horizon)$. The player-pessimal stable regret and player-optimal stable regret tend to coincide in many real-world markets, such as in unbalanced random matching markets \citep{ashlagi2017unbalanced} where the stable matching is essential unique. This is as well the case when players are globally ranked. In this work, we focus on the player-pessimal stable regret.} It is defined as follows for player~$\player{i}$: 
\begin{align}
	\label{eq:pess_regret}
	\pessregret{i}(\horizon) := \horizon \meanreward{i}{\pessmatch{i}} - \sum_{t = 1}^\horizon \EE \reward{i}{\actions{t}}{t}.
\end{align}

\lledit{The above notion of stable regret considers regret from the perspective of the players only, that is, we are primarily interested in how the players perform with respect to their stable arms over time. Focusing on the welfare of one side of the market is consistent with the stable matching literature, in particular that on school choice, where one side of the market (the schools) are said to have ``priorities'', rather than ``preferences'',  for the other side of the market (the students), and it is the students' welfare that is of primary interest \citep{abdulkadirouglu2003school, abdulkadiroglu2006changing}.\footnote{We thank a reviewer for pointing out this connection to the economics literature.} Recently, \citep{cen2021regret} studied fairness and social welfare in the context of matching markets. 
}

%\lnote{Should we discuss convergence to stability here? A bit weird because section 4 does not discuss that and the results we have there also don't imply convergence to stability.}
% 

%Recall, that we say a player $i$ prefers arm $j$ over arm $j^\prime$ when $\meanreward{i}{j} > \meanreward{i}{j^\prime}$. 

In order to fully specify the problem we need to clarify what information the players have access to. We consider the following decentralized setting:

\paragraph{Decentralized with Conflict Information} At each round, each player attempts to pull an arm, with the choice of arm based on only their rewards and observations from previous rounds. At the end of the round, all players can observe the winning player for each arm. They can see their own rewards only if they successfully pull an arm. They cannot see the rewards of other players.
We also assume that all players know, for each arm, which players are ranked higher than themselves.\footnote{This assumption allows for a  cleaner analysis of our algorithm. Our results can be generalized to the setting where players do not know this information initially because the arms know their own preferences and the conflicts between players are resolved deterministically. It is sufficient for each player to assume in the beginning that they are the most preferred player by every arm. Then, each lost conflict reveals which players are more preferred by which arms. This procedure would introduce at most $\numarms \numagents^2$ conflicts. 
 %It will only take any player one conflict with another player on any arm to know that the latter is ranked higher by that arm. 
	%It can be checked that if they were to run Algorithm~\ref{alg:ucb_ca_random} as is, each player will only incur an additional number of conflicts that is at most quadratic in the size of the market.
}

%This assumption is not limiting. If the information was not known a priori, because arms know their own preferences and because the conflicts are resolved deterministically, the players can determine the preferences of all the arms $\{\asucc{j}\}_{j\in \arms}$ in a fixed number of rounds that scales at most quadratically with $\numarms$ and $\numagents$. \todo{was confusing to Nikhil should add that this decentralized conflict learning phase requires players to randomly pull arms? Or remove this claim about WLOG. }

%\todo{Explain that if this is not known ahead of time, then we can have a decentralized conflict learning phase with complexity at most $\numagents^2\numarms$ that succeeds with high probability?}

%\paragraph{Notation} $\agentset$ is the set of players, $\armset$ is the set of arms, and $|\agentset| =\numagents$,  $|\armset| = \numarms$.  Let $\psucc{i}$ denote the strict preference ordering that player $\player{i}$ has over arms, and let $\asucc{j}$ denote the strict preference ordering that arm $\arm{j}$ has over players. %Throughout this section we assume the stable matching is unique. WLOG let the unique stable matching pairs be $(\player{i}, \arm{i})$. 

\section{Algorithm: Decentralized Conflict-Avoiding UCB}\label{sec:alg}

%In the model introduced by \citet{liu20competing} arms have preferences over players, that are expressed in the following way. When multiple players pull the same arm at a given time step only the most preferred player receives a stochastic reward. From a different perspective, we can also think of players having different levels of skill to obtain a reward from a given arm. This competition between players adds a level of complexity to attaining an efficient exploration-exploitation trade-off.  

In the single-player multi-armed bandit (MAB) the player must explore different arms in order to identify the arms with the highest mean payoff. At the same time, the player must keep selecting arms that seem to give high payoff in order to accumulate a large reward over time. The upper confidence bounds (UCB) algorithm offers an elegant solution to this exploration-exploitation dilemma. As the name suggests, UCB maintains upper confidence bounds on the arms' mean payoffs and selects the arm with the largest upper confidence bound. Then, the UCB algorithm updates the upper confidence bound corresponding to the selected arm according to the reward observed.

In the aforementioned decentralized model, however, a player cannot implement UCB obliviously of other players' actions given the possibility of conflicts. Let us discuss this issue from the perspective of player $\player{1}$. Suppose~$\player{1}$ chooses arm $\arm{1}$, and suppose player~$\player{2}$ chooses $\arm{1}$ at the same time. Then, if $\arm{1}$ prefers $\player{2}$ over $\player{1}$, a conflict arises and player~$\player{1}$ receives no reward. In addition to not receiving a reward, in this case, player $\player{1}$ does not learn anything new about the distribution of rewards offered by arm $\arm{1}$. Therefore, in the decentralized case players must balance exploration and exploitation while avoiding conflicts that they would lose. 

To see intuitively how $\player{1}$ can achieve such conflict avoidance let us assume that there are only two players and that all arms prefer $\player{2}$. Then, from the perspective of $\player{2}$, the problem is identical with the single-player MAB problem and therefore $\player{2}$ can achieve small regret by using the standard UCB method. Since $\player{2}$ aims to minimize their own regret,~$\player{2}$ will sample the arm that gives them the highest mean payoff most of the time. More precisely, there can be at most $\Ocal(\log(T))$ time steps when $\player{2}$
does not sample the best arm for themself.

On the other hand, 
$\player{1}$ must minimize the number of times they select the same arm as $\player{2}$ because they would lose the conflicts with $\player{2}$.  Because most of the time player $\player{2}$ chooses the best arm for themselves, the following simple heuristic allows player $\player{1}$ to avoid choosing the same arm as $\player{2}$ most of the time: player $\player{1}$ should not select the arm $\player{2}$ chose at the previous time step. 

It turns out that this conflict-avoidance heuristic, combined with the UCB method, gives rise to an algorithm that provably achieves small regret for all players. We call this method \emph{Decentralized Conflict-Avoiding Upper Confidence Bound}, or \emph{CA-UCB} for short, and detail it in Algorithm~\ref{alg:ucb_ca_random}. Before introducing our algorithm, let us first introduce some notation for the players' actions. We use $\attempt{i}(t)$ to  denote the player $\player{i}$'s attempted arm at time $t$, and $\pull{i}(t)$ to denote the player $i$'s successfully pulled arm at time~$t$. When the player fails to pull an arm successfully because of a lost conflict, we have $\pull{i}(t) = \emptyset$.

\begin{algorithm}[t]
	\caption{CA-UCB with random delays}\label{alg:ucb_ca_random}
	\begin{algorithmic}[1]
		\renewcommand{\algorithmicrequire}{\textbf{Input: } $\stayprob \in [0,1)$}
		\renewcommand{\algorithmicensure}{\textbf{Output: }}
		\REQUIRE 
		\FOR{$t = 1, \ldots, T $}
		\FOR{$i = 1, \ldots, N $}
		\IF{$t=1 $}
		\STATE 	Set upper confidence bound to $\infty$ for all arms.
		\STATE Sample an index $j \sim 1, \ldots,\numarms$ uniformly at random. Sets $\attempt{i}_t \leftarrow \arm{j}$.
		\ELSE
		\STATE Draw $\delaydraw{i}(t) \sim Ber(\lambda)$ independently.
		\IF{$\delaydraw{i}(t) = 0 $}
		\STATE Update plausible set $\potset{i}(t)$ for player $\player{i}$: \[\potset{i}(t) := \{ \arm{j}: \player{i} \asucc{j} \player{k} \text{ or } \player{i} = \player{k}, \text{ where } \pull{k}(t-1) = \arm{j}  \}.\]%an arm $a$ is plausible if of all the players who has won a conflict on $a$ against $\player{i}$, none of them chose $a$ at time $t-1$.
		\STATE Pulls $a \in \potset{i}(t)$ with maximum upper confidence bound. Sets $\attempt{i}_t \leftarrow a$.
		\ELSE
		\STATE Pulls $\attempt{i}_{t-1}$. Sets $\attempt{i}_t \leftarrow \attempt{i}_{t-1}$.
		\ENDIF
		\ENDIF
		\IF{$\player{i}$ wins conflict}
		\STATE Update upper confidence bound for arm $\attempt{i}_{t}$. 
		\ENDIF
		\ENDFOR
		\ENDFOR
		%	\STATE \textbf{return} $\x_T$
	\end{algorithmic}
\end{algorithm}

According to Algorithm~\ref{alg:ucb_ca_random}, at each time step $t$ each player $\player{i}$ independently samples a biased Bernoulli random variable $\delaydraw{i}(t)$ with mean $\stayprob \in [0,1)$. When $\delaydraw{i}(t)$ comes up $1$, the player chooses the same arm as they did at the previous time step. We will soon return to explain the rationale behind staying on the same arm as the previous time step with some probability. For now, let us focus on the case where $\delaydraw{i}(t)$ comes up $0$.

When the Bernoulli random variable $\delaydraw{i}(t)$ comes up $0$, the player constructs a \emph{plausible set} of arms that includes all arms except those that the player would not have been able to pull successfully at the previous time step. In other words, the player $\player{i}$ will consider an arm plausible, only if in the previous time step $t-1$, the arm was not pulled by a player that the arm strictly prefers to $\player{i}$. Then, the player chooses the arm in the plausible set with the highest upper confidence bound, which is updated as in the single-player UCB method. We formally define the upper confidence bound in Equation~\ref{eq:ucb} of Section~\ref{sec:global-players}.

We refer to the parameter $\stayprob$ as the \emph{delay probability}. When $\stayprob = 0$ the actions of the players that implement CA-UCB are deterministic functions of the history up to that point. This property has no impact on the algorithm's convergence when the players are globally ranked (i.e., all arms have the same preferences), as shown in Section~\ref{sec:global-players}. However, for more general preference structures, if all players implement CA-UCB with delay probability zero, they can enter into infinite loops. The following simple example showcases this failure mode. 

\begin{example}[2-player globally ranked arms]\label{ex:2player}
	Consider the following setting with two players and two arms:
	\begin{align*}
		\player{1}: \arm{1} \succ \arm{2} \quad\quad \arm{1}: \player{1} \succ \player{2} \\
		\player{2}: \arm{1} \succ \arm{2} \quad\quad \arm{2}: \player{2} \succ \player{1}.	
	\end{align*}
	In this case the unique stable matching is $(\player{1}, \arm{1}), (\player{2}, \arm{2})$.
\end{example}

Suppose both players in Example~\ref{ex:2player} implement CA-UCB with zero probability of delay. 
Through a random initialization of CA-UCB it is possible that both players select arm $\arm{1}$ at the first time step. Then, $\player{2}$ loses the conflict and at the next step will choose $\arm{2}$, which is the only arm in their plausible set. On the other hand, the UCB of player $\player{1}$ for arm $\arm{2}$ is positive infinity at this point because they have not pulled it yet.  Hence, $\player{1}$ attempts to pull $\arm{2}$ at the second time step. Since $\arm{2}$ prefers $\player{2}$, $\player{1}$ loses the conflict and their UCB for arm $\arm{2}$ remains infinite.
The same argument shows that both players will keep choosing the same arm, alternating between $\arm{1}$ and $\arm{2}$. As long as they stay in this cycle, both players experience a constant stable regret. We showcase another example of when deterministic conflict-avoiding might fail in Appendix~\ref{app:example}.

To break such cycles CA-UCB incorporates randomness via the delay probability. As we will see, for arbitrary preferences and delay probability $\stayprob \in (0,1)$, the CA-UCB algorithm achieves $\Ocal(\log(T)^2)$ regret, with the hidden constant depending on $\stayprob$, the gap between mean rewards, and the number of players and arms. 
On the other hand, the size of the regret that we obtain depends exponentially on the number of players, regardless of the choice of~$\stayprob$. %We believe that this poor dependence on the number of players to be a consequence of our specific analysis; numerical experiments suggest that the dependence is better in practice even in a general setting. 
We can obtain stronger results by making additional assumptions on the structure of preferences. In particular, if the players are globally ranked, then we obtain a polynomial dependence on the number of players; moreover, we obtain $\Ocal(\log(T))$ regret.  We begin with this specialized setting in Section~\ref{sec:global-players} and turn to the general case in Section~\ref{sec:general-prefs}.

% !TeX root = main.tex 

\section{Globally Ranked Players}
\label{sec:global-players}

%\todo{The results in section are complete.}

In this section, we prove regret bounds for the CA-UCB algorithm, Algorithm~\ref{alg:ucb_ca_random}, without random delays (i.e., with $\lambda = 0$). We assume all arms have the same preferences over players, whereas each player may have arbitrary preferences over arms. This preference structure is made precise in the following assumption.

\begin{assumption}[Globally ranked players]\label{assmp:global_players} 
	We assume the players are globally ranked: for any $\player{i}$, $\player{i'}$ where $i < i'$, and any arm 
	$\arm{j}$, we have $\player{i} \asucc{j} \player{i'}$.
\end{assumption}

In other words, more preferred players have lower indices. Under this assumption, there is a unique stable matching in the market. By re-indexing the arms we can assume without loss of generality that the stable player-arm pairs are $\{(\player{i}, \arm{i})\}_{i=1 }^{\numagents}$. Under such an indexing, the following critical property holds: for any player $\player{i}$ and any arm $\arm{j}$ with $j > i$, $\player{i}$ must prefer $\arm{i}$ over $\arm{j}$; that is, $\arm{i} \psucc{i} \arm{j}$. Also, since the stable matching is unique, there is a single notion of stable regret, that is, for any player $\player{k}$, we have $\regret{k}(\horizon) := \pessregret{k}(\horizon) = \optregret{k}(\horizon)$.

Our goal in this section is to prove an upper bound on the stable regret of a player, taking into account their ranking in the market. We use the following notation to denote the gaps in mean rewards of arms for players $\player{i}$, $\player{j}$:
\begin{equation}
\gap{i}{j} := \meanreward{i}{i} - \meanreward{i}{j} ~~\text{and}~~ \gap{i}{\emptyset} := \meanreward{i}{i}. 
\end{equation}
We use $\mingap^2 := \min_{i < j} |\gap{i}{j}|^2$ to denote the minimum squared gap.

\begin{theorem}[Stable regret under globally ranked players]\label{thm:global-players}
\label{thm:global}
Suppose each player runs Algorithm~\ref{alg:ucb_ca_random} with $\stayprob = 0$. The following 
regret bound holds for any player $\player{k}$ and any horizon $T \ge 2$: %, up to universal constants:
\begin{equation}
\regret{k}(\horizon)\leq 6 k^2  \left( \frac{\log \horizon}{\Delta^2} + 1\right) \cdot \bigg( (\numarms - k)\gap{k}{\emptyset}
		+  k \sum_{i: \arm{k} \psucc{k}\arm{i}} \gap{k}{i} \bigg).
\end{equation}
\end{theorem}

This result shows that the stable regret of any player in the market is logarithmic in the horizon~$\horizon$, matching the known lower bound for single-player stochastic bandits \citep{Lai85asymp}. Moreover, the regret scales cubically with the rank of the player and linearly with the number of arms.  It is useful to compare this result to the corresponding stable regret obtained by \citet{liu20competing} in the centralized setting, also under Assumption~\ref{assmp:global_players}:
\begin{align}
\regret{k}(\horizon) \leq 6 k \sum_{l = k + 1}^\numarms \left(\gap{k}{l} + \frac{\log \horizon}{\gap{k}{l}}\right).
\end{align}
We see that in the centralized setting, the dependence on the rank $k$ is linear instead of cubic. Moreover, the dependence on the reward gap is reduced to $\sum_{i > k} 1/\gap{k}{i}$, which matches the optimal dependence on the reward gaps in the classical single-player bandit problem \citep{Lai85asymp}. In the decentralized setting where players are globally ranked, \citet{Sankararaman20dominate} showed a instance dependent lower bound suggesting that the dependence on $1/\mingap^2$ cannot be improved upon in general. We further discuss lower bounds in Section~\ref{sec:dis}.

Before we proceed to the proof of Theorem~\ref{thm:global-players}, we introduce the following notation, and establish two technical lemmas. 
\begin{itemize}
	\item $\attempt{k}(t) \in [\numarms]$  is the arm attempted by $\player{k}$ at time $t$;
	\item	$\pull{k}(t)\in [\numarms] \cup \{\emptyset\}$ is outcome of $\player{k}$'s attempt at time $t$;
	\item $\numattempts{k}{i}{t}$ is the total number of attempts by $\player{k}$ of $\arm{i}$ up to time $t$;
	\item $\numpulls{k}{i}{t}$ is the total number of successful attempts by $\player{k}$ of $\arm{i}$ up to time $t$.
\end{itemize}
The following events are central to our analysis:
	\begin{equation}
	\label{eqn:def-of-Lambda-event}
		\errorpot{j}{l}{t}= \left\{\bar{A}^{(j)}(t) = a_l, \arm{j} \in S^{(j)}(t)\right\}.
	\end{equation}
In plain language, $\errorpot{j}{l}{t}$ denotes the event in which a player $\player{j}$ chooses to pull an arm $\arm{l}$ over a stable matching arm $\arm{j}$ that belongs to the plausible set at time $t$. 

The next lemma shows that if a player $\player{k}$ pulls a suboptimal arm $\arm{i}$ (with $i > k$) at time $t$, then there
	must be some same or better-ranked player $\player{j}$ (with $j \leq k$), who, though having its matching arm 
	$\arm{j}$ in their plausible set, chose to pull a suboptimal arm $\arm{l}$ (with~$l > j$) at some time $t^\prime$ between times $t-k$ and $t$. 

\begin{lemma}[Suboptimal pulls]\label{lem:global_subopt}
	For any player $\player{k}$ and arm $\arm{i}$ such that  $\arm{k} \psucc{k} \arm{i}$, %we have
	%the inclusion
	\begin{equation}
	\label{eqn:tracing-back-events}
	\left\{\pull{k}(t)= \arm{i}\right\} 
		\subseteq \errorpot{k}{i}{t} \bigcup \bigg(\bigcup_{1\le j < l \le k} \bigcup_{t-k \le t^\prime < t}  
		\errorpot{j}{l}{t'}\bigg).
	\end{equation}
%		\lnote{why not $\bigcup_{t-(k-j) \le t^\prime \le t} $? }
\end{lemma}
\begin{proof}
The key to the proof is the following observation. Suppose the event $\left\{\pull{k}(t) = \arm{i}\right\}$ takes place. Then, one of the two things must happen: 
		\begin{itemize}
		\item $\arm{k} \in \potset{k}(t)$, in which case the event $\errorpot{k}{i}{t}$ occurs by definition.
		\item $\arm{k}  \not\in \potset{k}(t)$, in which case some better-ranked player, say $\player{u}$ with $u < k$, must have pulled the arm $\arm{k}$ at time $t-1$ according to the definition of Algorithm~\ref{alg:ucb_ca_random}.
		\end{itemize}
	This observation translates to the following assertion:  for any player $\player{k}$ and arm $\arm{i}$ where $i \neq k$, we have
		\begin{equation}
		\label{eqn:key-in-the global_subopt}
			\left\{\pull{k}(t) = \arm{i}\right\} 
				\subseteq \errorpot{k}{i}{t} \bigcup  \bigg(\bigcup_{u < k} \left\{\pull{u}(t-1) = \arm{k}\right\}\bigg).
		\end{equation}
We can now prove the lemma by induction on $k$. 

	\emph{Base case $k = 1$:} This is trivially true, due to the fact that the top-ranked player $\player{1}$ has all the 
		arms in their plausible set at all times $t$, and thus, for any arm $\arm{i}$, 
		\begin{equation*}
			\left\{\pull{1}(t) = \arm{i}\right\} = \left\{\pull{1}(t) = \arm{i}, \arm{i} \in \potset{1}(t)\right\}
				= \errorpot{1}{i}{t}.
		\end{equation*}
	
\emph{Induction step:} We assume \eqref{eqn:tracing-back-events} for all $k < m$ and  prove it also holds for $k = m$. 
		Let arm $a_i$ be such that $\arm{m} \psucc{m} \arm{i}$. By equation~\eqref{eqn:key-in-the global_subopt}, 
		we have 
		\begin{equation}
		\label{eqn:next-hypothesis-two-step-before}
			\left\{\pull{m}(t) = \arm{i}\right\} 
				\subseteq \errorpot{m}{i}{t}\bigcup  \bigg(\bigcup_{u < m} \left\{\pull{u}(t-1) = \arm{m}\right\}\bigg).
		\end{equation}
		By our assumptions we know that $\arm{u} \psucc{u} \arm{m}$ when $u < m$. Consequently, 
		we can apply the induction hypothesis for player $\player{u}$, with $u < m$, and arm $\arm{m}$ and 
		time $t-1$, to obtain that 
		\begin{equation*}
			\left\{\pull{u}(t-1) = \arm{m}\right\}
				\subseteq \errorpot{u}{m}{t-1} \bigcup  \bigg(\bigcup_{1\le j < l \le u} \bigcup_{t-u \le t' < t-1} \errorpot{j}{l}{t'}
					\bigg).
		\end{equation*}
		Taking the union over $u < m$ on both sides yields the inclusion 
		\begin{equation}
		\label{eqn:next-hypothesis-one-step-before}
		\begin{split}
			\bigcup_{u < m}\left\{\pull{u}(t-1) = \arm{m}\right\}
				&\subseteq \bigcup_{1\le j < l \le m} \bigcup_{t-m \le t' < t} \errorpot{j}{l}{t'}.
		\end{split}
		\end{equation}
		By substituting equation~\eqref{eqn:next-hypothesis-one-step-before} into equation
		\eqref{eqn:next-hypothesis-two-step-before}, we obtain the conclusion. 
\end{proof}

The next lemma tells a similar story as Lemma~\ref{lem:global_subopt}; it shows that when $\player{k}$ has a conflict, there 
	must be some better player $\player{j}$, with $j < k$, who chooses to pull a suboptimal arm $\arm{l}$ at 
	some time $t^\prime$ between times $t-k$ and $t$ although they have the matching arm $\arm{j}$ in  their plausible set. 

\begin{lemma}[Conflicts]\label{lem:global_conf}
	For any player $\player{k}$, we have the inclusion
	\begin{equation}
	\label{eqn:tracing-back-events-conf}
	\left\{\pull{k}(t) = \emptyset\right\} 
		\subseteq \bigcup_{\substack{1\leq j < k \\ j < l \le \numarms}} \bigcup_{t-k \le t^\prime \le t}  \errorpot{j}{l}{t^\prime}.
	\end{equation}
%	\lnote{why not $\bigcup_{t-(k-j) \le t^\prime \le t} $? }
\end{lemma}
\begin{proof}
Player $\player{k}$ can have a conflict on any of the arms $\arm{1}$, $\arm{2}$, \ldots, $\arm{\numarms}$. We have 
\begin{align*}
\left\{\pull{k}(t) = \emptyset\right\} = \bigcup_{l = 1}^\numarms \left\{\pull{k}(t) = \emptyset, \attempt{k}(t) = \arm{l}\right\}. 
\end{align*}
For all $m\geq k$ we observe that $\player{k}$ can have a conflict on $\arm{l}$ only if 
there is a player $\player{j}$ with $j < k$ who successfully pulls arm $\arm{m}$ at time $t$. In this case we have 
\begin{align*}
\left\{\pull{k}(t) = \emptyset, \attempt{k}(t) = \arm{m}\right\} 
		\subseteq \bigcup_{j < k} \left\{\pull{j}(t) = \arm{m} \right\}.
\end{align*}
We can then apply Lemma~\ref{lem:global_subopt} to each event $\left\{\pull{j}(t) = \arm{m} \right\}$.

We now have to analyze the events $\left\{\pull{k}(t) = \emptyset, \attempt{k}(t) = \arm{m}\right\}$ with $m < k$. Since 
\begin{align*}
\left\{\pull{k}(t) = \emptyset, \attempt{k}(t) = \arm{m}\right\} \subseteq \left\{\attempt{k}(t) = \arm{m}\right\},
\end{align*}
it suffices to prove by induction that
\begin{align}
\bigcup_{m = 1}^{k - 1}\left\{\attempt{k}(t) = \arm{m}\right\} \subseteq \bigcup_{\substack{1\leq j < k \\ j < l \le \numarms}} \bigcup_{t-k \le t^\prime \le t}  \errorpot{j}{l}{t^\prime}.\label{eq:superopt-attempts}
\end{align}

The base case $k = 1$ is obvious since the left-hand side is the empty set. Now, we assume the induction hypothesis holds for all $k < k'$ and we prove it for $k=k'$. If $\left\{\attempt{k'}(t) = \arm{m}\right\}$ holds, we know that $\player{m}$ at time $t - 1$ did not attempt to pull $\arm{m}$. They either attempted to pull an arm $\arm{m^\prime}$ with $m^\prime > m$ or with $m^\prime < m$. In the former case, the induction step follows from Lemma~\ref{lem:global_subopt}. 
In the latter case, we can apply our induction hypothesis. The result follows. 
\end{proof}

The final ingredient we need to prove Theorem~\ref{thm:global-players} is the UCB argument for a single player. This is given in the following display. For completeness, we provide an elementary proof in Appendix~\ref{app:general-prefs}.

\renewcommand{\event}{\mathcal{E}}
\begin{lemma}[UCB bound]
	\label{lem:ucb_bound}
	Suppose we use the following upper confidence bounds in Algorithm~\ref{alg:ucb_ca_random}:
	\begin{align}
	\label{eq:ucb}
	\ucb{i}{j}{t} = \left\{
	\begin{array}{ll}
	\infty & \text{if } \numpulls{i}{j}{t}= 0,\\
	\hatmeanreward{i}{j}{t} + 
	\sqrt{\frac{3\log t}{2\numpulls{i}{j}{t-1}}} & \text{otherwise.} 
	\end{array}\right.
	\end{align}
	
	Then, for any player $\player{i}$, arms $\arm{j}, \arm{k}$, such that $\arm{j} \prec_i \arm{k}$, we have, for $\horizon > 0$: 
	\begin{equation*}
	\sum_{t=1}^\horizon \P \left(\left\{\ucb{i}{j}{t} > \ucb{i}{k}{t} \right\} \cap \left\{ \pull{i}(t) = j \right\}\right) \le \frac{6}{\Delta^2}\log (T) + 6.
	\end{equation*}
\end{lemma}

\begin{proof}[Proof of Theorem~\ref{thm:global-players}]
We bound the regret of player $\player{k}$. By definition, their regret is
\begin{equation}
\label{eqn:simple-very-starting-point}
\regret{k}(\horizon) \le \gap{k}{\emptyset} \cdot \E \left[\numpulls{k}{\emptyset}{\horizon}\right]
	+ \sum_{i: \arm{k} \psucc{k}\arm{i}} \gap{k}{i}\cdot \E\left[\numpulls{k}{i}{\horizon}\right],
\end{equation}
%In this case, the pessimal stable regret is equal to the optimal stable regret.
where, because of our assumption on the indexing of arms, the last summation can also be written simply as a sum over all $i \in \{k + 1, \ldots, \numarms\}$.

\paragraph{Upper bounding $\E[\numpulls{k}{i}{\horizon}]$.} By definition,
	\begin{equation}
	\label{eqn:regret-simple-starting-point}
	\begin{split}
		\E \left[\numpulls{k}{i}{\horizon}\right] %&=  \E \left[\sum_{t=1}^\horizon \indic{\pull{k}(t) = \arm{i}}\right] = 
			&= \sum_{t =1}^\horizon \P\left( \pull{k}(t) = \arm{i}\right).
	\end{split}
	\end{equation}
	We now bound the probability $\P\left( \pull{k}(t) = \arm{i}\right)$ for each $t$.  
	Lemma~\ref{lem:global_subopt} yields
	\begin{equation}
	\label{eqn:global_subopt-immediate}
	\P(\pull{k}(t) = \arm{i}) \le \P( \errorpot{k}{i}{t}) + 
		\sum_{1\le j < l \le k} \sum_{t-k \le t' < t} \P ( \errorpot{j}{l}{t^\prime} ).
	\end{equation}
	Summing from $t= 1$ to $\horizon$, and using equation~\eqref{eqn:regret-simple-starting-point}, 
	we obtain the bound 
	\begin{equation}
	\label{eqn:telescope-result-simple}
	\begin{split}
	\E \left[\numpulls{k}{i}{\horizon}\right] %&= \sum_{1\le t\le \horizon} \P(\bar{A}^{(k)}(t) = \arm{i}) \\
		&\le \sum_{1 \le t \le T}  \P(\errorpot{k}{i}{t})
			+ \sum_{1\le t\le \horizon} \sum_{1\le j < l \le k} \sum_{t-k \le t' \le t} 
			\P (\errorpot{j}{l}{t^\prime} ) \\
		& \le \sum_{1 \le t \le T}  \P(\errorpot{k}{i}{t})+ 
			(k+1) \sum_{1\le j < l \le k}  \sum_{1\le t \le \horizon} 
			\P (\Lambda^{(j)}_l(t)).
	\end{split}
	\end{equation}
%	\lnote{The factor should be $(k+1)$ if you're using lemma 4.}

Recall that for all players $\player{j}$ and arms $\arm{l}$ with $l > j$, and time $t > 0$, \[\Lambda^{(j)}_l(t) \subseteq \left\{ \ucb{j}{l}{t} > \ucb{j}{i}{t} \right\} \cap \left\{ \pull{j}(t) = l \right\}.\] Therefore, using Lemma~\ref{lem:ucb_bound}, we can show that
the following upper bound 
	holds: 
	\begin{equation}
	\label{eqn:UCB-type-upper-bound}
		\sum_{1\le t' \le \horizon} \P (\Lambda^{(j)}_l(t')) %= 
		% \sum_{1\le t' \le \horizon} 
		%	\P \left(\bar{A}^{(j)}(t') = a_l, \arm{j} \in S^{(j)}(t')\right)
				\leq  6 \left(\frac{\log \horizon}{|\gap{j}{l}|^2} + 1\right).
	\end{equation}
	Substituting equation~\eqref{eqn:UCB-type-upper-bound} into 
	equation~\eqref{eqn:telescope-result-simple} yields the bound 
	\begin{equation}
	\label{eqn:global_subopt-end}
		\E \left[\numpulls{k}{i}{\horizon}\right] 
			\leq 6  \bigg(\frac{\log \horizon}{|\gap{k}{i}|^2} + 1\bigg)
			+ 6 (k+1) \sum_{1 \le j < l \le k} \bigg(\frac{\log \horizon}{|\gap{j}{l}|^2} + 1\bigg)
		 \le 6 k^3 \left( \frac{\log \horizon}{\Delta^2} + 1\right).
	\end{equation}
	Recall that $\mingap^2 = \min_{i \neq j} |\gap{i}{j}|^2$.

\paragraph{Upper bounding $\E \left[\numpulls{k}{\emptyset}{\horizon}\right]$.} By definition, 
	\begin{equation}
	\label{eqn:regret-simple-starting-point-two}
	\begin{split}
		\E \left[\numpulls{k}{\emptyset}{\horizon}\right] %&=  \E \left[\sum_{t=1}^\horizon \indic{\pull{k}(t) = \arm{i}}\right] = 
			= \sum_{t =1}^\horizon \P\left( \pull{k}(t) = \emptyset\right).
	\end{split}
	\end{equation}
	Lemma~\ref{lem:global_conf} and a derivation mutatis mutandis to the argument from
	equation~\eqref{eqn:global_subopt-immediate} to~\eqref{eqn:global_subopt-end}
	yields
	\begin{equation}
	\label{eqn:global_conf-end}
	\begin{split}
		\E \left[\numpulls{\emptyset}{i}{\horizon}\right] 
			\leq 6 (k+1) \sum_{\substack{1\leq j <k \\ j < l \le \numarms}} \bigg(\frac{\log \horizon}{|\gap{j}{l}|^2} + 1\bigg)
			\le 6 k^2 (\numarms-k) \left( \frac{\log \horizon}{\Delta^2} + 1\right). 
	\end{split}
	\end{equation}
Substitute \eqref{eqn:global_subopt-end} and~\eqref{eqn:global_conf-end}
into \eqref{eqn:simple-very-starting-point} to complete the proof of the theorem.
\end{proof}

\section{Arbitrary Preferences on Both Sides of the Market}\label{sec:general-prefs}

%\todo{The results in section are complete.}

In this section, we analyze the convergence of Algorithm~\ref{alg:ucb_ca_random} under arbitrary preference lists for both sides of the market. Note that in this setting, the stable matching may not be unique. We consider throughout the randomized version of Algorithm~\ref{alg:ucb_ca_random}, with delay probability~$\stayprob > 0$. 

Without the assumption of shared preferences among the arms, the analysis of the convergence of Algorithm~\ref{alg:ucb_ca_random} becomes more challenging. In fact, it is not obvious that Algorithm~\ref{alg:ucb_ca_random}, or any other algorithm, can achieve sublinear player regret against the pessimal stable matching for any set of preferences. As seen in Example~\ref{example:pessimal_regret} in Appendix~\ref{app:example}, decentralized coordination among players can be difficult even in small markets with only three players. In order to prove the regret bound in Section~\ref{sec:global-players}, we relied heavily on the structure conferred by the global ranking of players. Without this particular structure, we have to appeal to more general results about stable matching. This generality also comes at a cost: the regret bound we prove in this section is polylogarithmic in the horizon and has an exponential dependence on the number of players.

Before introducing the main result, we first present some essential notation. Recall that $\agentset = \{\player{i}\}_{i=1}^\numagents$ denotes the set of players, and  $\armset = \{\arm{i}\}_{i=1}^\numarms$ denotes the set of arms. We denote the attempted actions (i.e., arms) at time $t$ as
\begin{equation*}
	\match_t: \agentset \mapsto \armset, \text{ where } \match_t(\player{i}) \defeq \attempt{i}(t).
\end{equation*} 
We note that $\match_t$ in general does not have to be a matching between players and arms, because two or more players may attempt to pull the same arm. %Nevertheless, $\match_t$ induces a matching (an injective map) between players and arms once the conflicts are resolved, 
However, whenever there are no conflicts, $\match_t$ is indeed a matching (an injective map) between players and arms, so we can distinguish the set of attempted actions that coincide with a stable matching. We thus refer to $\match :\agentset \mapsto \armset$ as \emph{stable} if $\match$ indeed coincides with a stable matching between players and arms.

We denote the set of stable attempted actions as \[\mstar := \{M \mid  M:\agentset \mapsto \armset, M~\text{is stable}\}.\] Let $\mingap = \min_{i,j,k} |\meanreward{i}{j} - \meanreward{i}{k}|$ denote the minimum reward gap between any two arms for any player. We also define the constant $\varepsilon := (1-\stayprob)\stayprob^{N-1}$, which depends on the delay probability $\stayprob$.

Our goal in this section is to prove the following upper bound on the probability that the market is in an unstable configuration when running the algorithm. More formally, we bound the sum, over $t$, of probabilities that the attempted actions at time $t$ yield an unstable matching. Understanding how this quantity depends on the horizon and various problem parameters enables us to provide a general regret bound for Algorithm~\ref{alg:ucb_ca_random}.

\begin{theorem}[Convergence to stability of Algorithm \ref{alg:ucb_ca_random} for arbitrary preferences]\label{thm:regret_general_pref}
	Let $\numagents, \numarms \geq 2$, $T\geq 2$, and suppose we run Algorithm \ref{alg:ucb_ca_random} with delay probability $\stayprob \in (0,1)$. Then,
	\begin{equation}\label{eq:general_pref_bound}
		\sum_{t=1}^T\Pr(\match_t \neq \mstable) \le 24 \cdot \frac{\numagents^5 \numarms^2}{\varepsilon^{\numagents^4 + 1}} \log(T) \left(\frac{1}{\Delta^2}\log (T) + 3\right).
	\end{equation}
\end{theorem}

As a corollary of Theorem~\ref{thm:regret_general_pref}, we have the following upper bound on the pessimal stable regret of any player.

\begin{corollary}[Pessimal stable regret of Algorithm \ref{alg:ucb_ca_random} for arbitrary preferences]\label{cor:general_regret}
	The following inequality holds for the agent-pessimal regret of player $\player{k}$ up to time $\horizon$:
	\[	\pessregret{k}(\horizon) \le  24 \cdot 
	\max_{\arm{\ell}\in\armset}\pessgap{k}{\ell} \left( \frac{\numagents^5 \numarms^2}{\varepsilon^{\numagents^4 + 1}} \log(T) \left(\frac{1}{\Delta^2}\log (T) + 3\right) \right), \]
	where $\pessgap{k}{\ell}=\max\{ \meanreward{k}{\pessmatch{k}} - \meanreward{k}{\ell},\meanreward{k}{\pessmatch{k}} \} $. 
\end{corollary}

In short, we find that the stable regret of Algorithm~\ref{alg:ucb_ca_random} is $\Ocal((\log \horizon)^2)$. Unlike in previous sections where we derived player-specific stable regret bounds that depended on the ranking of the player, or the ranking of their stable arm, in the current setting the players have no particular ranking. Corollary~\ref{cor:general_regret} is derived from a  general bound on the probabilities that the matching of the entire market is unstable.

\paragraph{Proof sketch} We begin by sketching the main ideas in the proof of Theorem~\ref{thm:regret_general_pref}. There are two main technical ingredients that are new to the current section: the first is the observation that in the event that each player's UCB rankings of the arms in their plausible set are correct (colloquially we refer to this event as ``no statistical mistakes"), and the previous matching was stable, then running one step of Algorithm~\ref{alg:ucb_ca_random} will preserve the stability of the matching with probability one. This is established in Lemma~\ref{lem:still-stable}. Therefore, if the matching at time $t$ is unstable, it must be either be that some player had incorrect UCB rankings, or there were no statistical ranking mistakes but the matching at time $t-1$ was unstable. 

In Lemma~\ref{lem:unstable-event-sets}, we generalize this statement to consider histories of arbitrary length. That is, if a matching at time $t$ is unstable, it must either be that some player had incorrect UCB rankings over the last $\history$ time steps, or there were no ranking mistakes  in all the last $\history$ time steps but the matchings reached were unstable. 

As in Section 5, we know how to upper bound the probability that a player had incorrect UCB rankings when running Algorithm~\ref{alg:ucb_ca_random} with $\lambda > 1$. Recall that this entailed a simple adaptation of the single-player UCB argument (Lemma~\ref{lem:ucb_bound}). The new problem we face is that of controlling the probability that there were no ranking mistakes but the matchings in all the last $\history$ time steps were unstable. It turns out that a classical result from the stable matching literature \citep{hernan95paths} gives us a way to argue that this probability is exponentially small in the length of the history considered (Lemma~\ref{lem:exponential-prob}). Intuitively, we are using the fact that Algorithm~\ref{alg:ucb_ca_random}, when there are no ranking mistakes, is essentially resolving \emph{blocking pairs}---pairs of players and arms that would prefer to be matched with each other over their current matches---in a randomized fashion, but following an order that is consistent with player preferences (Lemma~\ref{lem:bp_positive}). This is crucial for establishing that Algorithm~\ref{alg:ucb_ca_random} will always reach a stable matching with enough steps, as long as there are no ranking mistakes.

Finally, our analysis needs to balance the tradeoff inherent in the choice of the length of history considered, $\history$. If we consider a longer history length, there can be many ranking mistakes made in this window, hence contributing to a higher probability of an unstable matching. On the other hand, a longer history length with no ranking mistakes means that there is a higher probability that a stable matching can be reached. By choosing $\history$ to depend on the time step $t$, we are able to achieve a $\log(\horizon)^2$ dependence on the horizon $\horizon$ in the final bound \eqref{eq:general_pref_bound}.

Before presenting the technical lemmas, we first rigorously define the events of interest that were alluded to in the proof sketch.
\begin{enumerate}
	\item Let $E_t$ denote the event that, for every player, the arm that has the highest mean reward 
	in their plausible set coincides with the arm with the highest UCB in their plausible set at time~$t$:  
	\begin{equation}\label{eq:good-ucb-event}
		E_t \defeq \bigcap_{\player{i} \in \agentset} \bigg\{\argmax_{\arm{j}\in \potset{i}(t)}\meanreward{i}{j} 
		= \argmax_{\arm{j} \in \potset{i}(t)} \ucb{i}{j}{t}\bigg\}.
	\end{equation}
	Let $E^c_t$ denote the complement of this event.
	\item Let $F_{j,k}^{(i)}(t)$ denote the event that player $\player{i}$'s UCB for arm $\arm{j}$ is greater than their UCB for arm $\arm{k}$ at time $t$:
	\begin{equation*}
		F_{j, k}^{(i)}(t) = \left\{\ucb{i}{j}{t} >\ucb{i}{k}{t}\right\}.
	\end{equation*}
	%\item Let $\mingap = \min_{i,j,k} \gaptrip{i}{j}{k}$, the minimum reward gap between any two arms for any player.
	%\item Let $\varepsilon = (1-\epsilon)\epsilon^{N-1}$, where $\epsilon$ denotes the random delay probability in Algorithm~\ref{alg:ucb_ca_random}. 
\end{enumerate}

The following lemma shows that if the current matching is stable, then one step of Algorithm~\ref{alg:ucb_ca_random} under the event defined in \eqref{eq:good-ucb-event} preserves the stability of the current matching.

\begin{lemma}[Preservation of Stability]
	\label{lem:still-stable}
	Assume $\match_t \in \mstar$. Then $\match_{t+1} \in \mstar$ on the event~$E_{t+1}$. 
\end{lemma}
\begin{proof}
	We show $\match_{t+1} = \match_t$ on event $E_{t+1}$. Let $\match=\match_t$.  Assume $E_{t+1}$ happens. 
	Suppose, for a contradiction, that some player $p$ attempts an arm $a \neq \match(p)$ at time $t+1$. Let $p'$ be the player that $a$ is matched to at time $t$, that is, $M(a) = p'$, if $a$ is matched at time $t$, and let $p' = \emptyset$, otherwise.
	Note that since $p$ is matched with $\match(p)$ at time $t$, $\match(p)$ must belong to the plausible 
	set of $p$ at time $t+1$ by definition of the algorithm. Since $p$ attempts $a \neq \match(p)$ at $t+1$, 
	this implies that (i) $p$ truly prefers $a$ over $\match(p)$ by definition of $E_{t+1}$ and (ii) $a$ truly 
	prefers $p$ over $p'$, since $a$ must be in the plausible set of $p$ at time $t+1$. Thus $(p, a)$ are a blocking pair for the 
	matching $\match$, contradicting the assumption that $\match \in \mstar$. Thus we have shown 
	$\match_{t+1} = \match_t = \match \in \mstar$.
\end{proof}

In the next lemma, we apply Lemma~\ref{lem:still-stable} repeatedly to show that the event that the current matching is unstable can be decomposed into prior events that occurred up to $K$ steps in the past. Specifically, if the current matching is unstable, then either the UCB ranking of arms were wrong at some point in the history of length $K$ (that is, \eqref{eq:good-ucb-event} was false), or the matching was unstable for $K$ consecutive steps even though \eqref{eq:good-ucb-event} was true in all $K$ steps.

%\todo{$0 \le K < t$ below only works if indexing starts from 0. If indexing starts from 1, we need $0 
%	\le K<t-1$. Change all indexing to start from 1.}

\begin{lemma}[Inclusion for unstable matching event]\label{lem:unstable-event-sets}
	We have the following inclusion that holds for any $0 \le K < t-1$: 
	\begin{equation*}
		\left\{\match_t \not\in \mstar \right\} \subseteq \left(\bigcup_{s=0}^{K} E_{t-s}^c\right) \bigcup 
		\left(\bigcap_{s=0}^K (E_{t-s} \cap \left\{{\match_{t-s-1} \notin \mstar}\right\}) \right).
	\end{equation*}
\end{lemma}
\begin{proof}
	This is an immediate consequence of Lemma~\ref{lem:still-stable}. In fact, Lemma~\ref{lem:still-stable} shows 
	\begin{equation}
		\label{eqn:master-in}
		\left\{\match_t \not\in \mstar\right\} \subseteq E_t^c \cup (E_t \cap \{\match_{t-1} \notin \mstar\}). 
	\end{equation}
	This shows Lemma~\ref{lem:unstable-event-sets} holds for $K = 0$. 
	A simple induction argument shows that Lemma~\ref{lem:unstable-event-sets} holds for general $K > 0$, $K<t-1$.
	%Formally, we start by
	%\begin{equation}
	%\label{eqn:induct-one}
	%	\left\{m_t \not\in M^*\right\} \subseteq  \left(\bigcup_{s=0}^{K} E_{t-s}^c\right) 
	%		\cup \left(\bigcap_{s=0}^K E_{t-s} \cap \left\{m_t \not\in M^*\right\}  \right).
	%\end{equation}
	%Note $E_t \cap \left\{m_t \not\in M^*\right\}  \subseteq 
	%	\{m_{t-1} \notin M^*\}$ by equation~\eqref{eqn:master-in}. Hence, an induction gives 
	%\begin{equation}
	%\label{eqn:induct-two}
	%	\left(\bigcap_{s=0}^K E_{t-s} \cap \left\{m_t \not\in M^*\right\}  \right)
	%	\subseteq 
	%	\left(\bigcap_{s=0}^K (E_{t-s} \cap \left\{{m_{t-s-1} \notin M^*}\right\}) \right).
	%\end{equation}
	%Substitute equatoin~\eqref{eqn:induct-two} into equation~\eqref{eqn:induct-one}. We obtain the desired 
	%Lemma~\ref{lem:unstable-event-sets}.
\end{proof}

Lemma~\ref{lem:unstable-event-sets} suggests that in order to derive an upper bound on the probability that $m_t$ is unstable, we can separately bound the probabilities of the event that the UCB ranking of arms has an error, %$E_t^c$ 
and the event %$\bigcap_{s=0}^K (E_{t-s} \cap \left\{{m_{t-s-1} \notin M^*}\right\}).$ 
that the matching was unstable for $K$ consecutive steps even though UCB rankings were correct in all $K$ steps. 
The following lemma addresses the former.

\begin{lemma}[Probability of ranking error event]
	\label{lemma:bound-E-t-c}
	The following inequality holds for any $t > 0$: 
	\begin{equation*}
		\P(E_t^c) \le \varepsilon^{-1} \cdot \sum_{(i, j, k), : \arm{j} \prec_i \arm{k}} \P(F_{j, k}^{(i)}(t)\cap  \pull{i}(t) = j ).
	\end{equation*}
\end{lemma}

\begin{proof}
	The key is the following observation. That $E_t^c$ happens implies the existence of some player $\player{i}$ 
	and arms $\arm{j}, \arm{k}$ in their plausible set at time $t$, such that while the arm $a_j$ achieves the highest UCB with respect to 
	player $\player{i}$, the player truly prefers arm $\arm{k}$ over $\arm{j}$, Hence, this implies 
	\begin{equation*}
		\P(E_t^c) \le \sum_{(i, j, k), : \arm{j} \prec_i \arm{k}} \P(\ucb{i}{j}{t} >\ucb{i}{k}{t}\cap \{j = \argmax_{j'} \ucb{i}{j'}{t}\} ).
	\end{equation*}
	Recall $F_{j, k}^{(i)}(t) =\{ \ucb{i}{j}{t} >\ucb{i}{k}{t}\}$. 
	Lemma~\ref{lemma:bound-E-t-c} now follows if we can show 
	\begin{align*}
		\begin{split}
			& \P(F_{j, k}^{(i)}(t) \cap \{j = \argmax_{j'} \ucb{i}{j'}{t}\} ) 
			\le \varepsilon^{-1}  \cdot\P(F_{j, k}^{(i)}(t) \cap  \pull{i}(t) = j ).
		\end{split}
	\end{align*}
	To see this, note that the player $\player{i}$ will successfully pull $\arm{j}$ if player $\player{i}$ doesn't draw a 
	random delay and all the rest of the players draw the random delay (meaning they all
	attempt the same arm as they attempted in the last round). By independence of the random draws, 
	this event happens with probability at least $\varepsilon = (1-\stayprob)\stayprob^{\numagents - 1}$.
\end{proof}

Having established this lemma, we can now easily apply the UCB argument as given in Lemma~\ref{lem:ucb_bound} to bound the relevant quantity, $\sum_{t=1}^T\P(E_t^c)$.

%	\lnote{The following lemma crucially assumes that all possible conflicts have already occurred, so that the plausible sets are correct. Otherwise, the $(p_i, a_j)$, where $a_j$ is the arm with the highest UCB in the plausible set of $p_i$, does not have to be a blocking pair. So currently, this part is not yet rigorous, if the players don't start off knowing their conflicts for each arm. It's not clear to me how to do this? One observation that might be helpful: if in any round, the arm attempted is not truly plausible, then there is a positive probability of conflict in that next round. 		
%		Perhaps we can consider the following math problem. Suppose the events $\{B_t\}$ are such that $\{B_t = 1 \} \cap \{ Ber(\epsilon)=1\} \implies B_{t+1}, B_{t+2} ... = 0 $. What is $\sum_t \Pr(B_t=1)$? In this case, $B_t$ can be the event that player $\player{i}$ attempts $\arm{j}$ at $t$  when some other player $\player{k}$ such that $\player{k} \succ_j \player{i}$ pulled $\arm{j}$ at $t-1$. }

We proceed to analyze the probability of the event that the matching was unstable for $K$ consecutive steps even though UCB rankings were correct in all $K$ steps. Essentially, this requires us to establish how quickly the decentralized conflict-avoiding procedure converges to a stable matching when there are no statistical errors in the rankings of arms. To do so, we invoke a result from the stable matching literature \citep{hernan95paths}. First, we introduce the notion of a blocking pair that is \emph{player-consistent}.

\begin{definition}[Player-consistent blocking pair]
	A blocking pair $(\player{i}, \arm{j})$ in a matching $\mu$ is \emph{player-consistent} if 
	\begin{equation}\label{eq:playerconsistentbp}
		\arm{j} \psucc{i} \arm{k}~~\text{for any $k$ such that $(\player{i}, \arm{k})$  is a blocking pair in $\mu$}.
	\end{equation}
\end{definition}

In other words, if player $\player{i}$ most prefers the $\arm{j}$ out of all the arms that prefer $\player{i}$ over the player that they are matched to in $\mu$, then the blocking pair $(\player{i}, \arm{j})$ is player-consistent. Notice that in Algorithm~\ref{alg:ucb_ca_random}, at time $t$, if the UCB rankings are accurate, then each player $\player{i}$ (who did not draw a random delay) will attempt precisely the arm $\arm{j}$ where $(\player{i}, \arm{j})$ is a player-consistent blocking pair in the matching $\mu$ induced by the previous attempted actions $m_{t-1}$, by the definition of the plausible set. 

We also require the following definition of \emph{resolving} a blocking pair, in the context of running one step of Algorithm~\ref{alg:ucb_ca_random}.

\begin{definition}[Resolution of blocking pair]
	Given attempted actions $m_t \notin M^* $ and a blocking pair $(\player{i}, \arm{j})$  in the matching induced by $m_t$, we say that $m_{t+1}$ is obtained by resolving $(\player{i}, \arm{j})$, if $m_{t+1}(\player{i} ) = a_j$ and $m_{t+1}(p) =  m_t(p)$ for all $p \in \agentset, p \neq \player{i}$.
\end{definition}

We are ready to establish a key result---that there is a strictly positive probability that a single player-consistent blocking pair is resolved in one step of Algorithm~\ref{alg:ucb_ca_random}.

\begin{lemma}[Positive probability of resolving a single blocking pair]\label{lem:bp_positive}
	Assume $m_{t-1}$ is unstable. Let $(\player{i}, \arm{j})$ be a blocking pair in $m_{t-1}$ that 
	is player-consistent. Condition on the event~$E_t$. Then, 
	with probability at least $\varepsilon = (1-\stayprob)\stayprob^{\numagents -1}$, 
	$(\player{i}, \arm{j})$ is the only blocking pair to be 
	resolved at time $t$, i.e., 
	\begin{equation}
		\label{eqn:suffice-bp-pos}
		\P\left(\match_t(\player{i})= \arm{j}, m_{t}(p) =  m_{t-1}(p) ~\forall p \in \agentset, p \neq \player{i}\mid E_t\right) \ge \varepsilon. 
	\end{equation}
\end{lemma}

\begin{proof}
	Assume $E_t$ holds. Let $(\player{i}, \arm{j})$ be any blocking pair that is player-consistent. 
	First, we show $\player{i}$ has probability at least $\stayprob$ of pulling the arm $\arm{j}$ conditioned 
	on all of the other players attempting the same arm as they pulled at time $t$. Indeed, since 
	$(\player{i}, \arm{j})$ is a blocking pair of $m_{t-1}$, it means that $\arm{j}$ is in the plausible set of 
	$\player{i}$ at time $t$. As $E_t$ occurs, and $\arm{j}$ is the top choice among all the arms in $\player{i}$'s plausible set, the player $\player{i}$ has probability at least 
	$\stayprob$ of attempting $\arm{j}$, and will be successful if all other players stay on the same arm 
	as they pulled at time $t$. Second, independently, each of the rest of the $N-1$ players have probability at least 
	$(1-\stayprob)$ of attempting the same arm that they attempted at time $t-1$. Together, this proves 
	equation~\eqref{eqn:suffice-bp-pos}. 
\end{proof}

Now we can finally show that the event 
that the matching was unstable for $K$ consecutive steps even though UCB rankings were correct in all $K$ steps happens with a probability that is exponentially small in $K$, as stated formally in the lemma below.

\begin{lemma}[Probability of not reaching a stable matching]\label{lem:exponential-prob}
	For any $0 \le K < t-1$, the following inequality holds:
	\begin{equation} \label{eq:expon-prob-unstable}
		\P\left(\bigcap_{s=0}^K \left(\{m_{t-s-1} \not\in M^* \} \cap E_{t-s} \right) \right) \le 
		(1-\varepsilon^{N^4})^{\lfloor K/\numagents^4\rfloor}.
	\end{equation}
\end{lemma}
\begin{proof}
	The result is a direct consequence of Lemma~\ref{lem:bp_positive} and the theorem below. 
	\begin{theorem}[Theorem 4.2 in \citet{hernan95paths}]
		\label{thm:matching-resolving}
		Given any unstable matching $\mu_0$, there exists a sequence of blocking pairs of length at most $N^4$ such that 
		resolving the sequence of blocking pairs reaches a stable matching. Moreover, this sequence of blocking pairs 
		results from resolving blocking pairs in a \emph{player-consistent order}, that is, any blocking pair 
		$(\player{i}, \arm{j})$ resolved in the current matching $\mu$ is player-consistent with respect to the matching~$\mu$.
	\end{theorem}
	We now prove Lemma~\ref{lem:exponential-prob} using Lemma~\ref{lem:bp_positive} and 
	Theorem~\ref{thm:matching-resolving}.
	\begin{enumerate}
		\item We first show Lemma~\ref{lem:exponential-prob}  holds when $K = N^4$. Let $E = \cap_{s=0}^K E_{t-s}$. Condition on 
		the event that $E$ happens. Condition on the matching $\mu = m_{t-K-1}$. 
		By Theorem~\ref{thm:matching-resolving} and Lemma~\ref{lem:bp_positive}, we know that with probability 
		at least $\varepsilon^{N^4}$, %all the blocking pairs in $\mu$ can be resolved 
		a stable matching will be reached within $N^4$ steps of the 
		algorithm. Since this holds for arbitrary $\mu = m_{t-K-1}$, we obtain 
		\begin{equation*}
			\P\left(\bigcap_{s=0}^K \{m_{t-s-1} \not\in M^* \} \mid E \right) \le 1-\varepsilon^{N^4}. 
		\end{equation*}
		Thus, we have 
		\begin{equation*}
			\P\left(\bigcap_{s=0}^K \left(\{m_{t-s-1} \not\in M^* \} \cap E_{t-s} \right) \right)  \le 
			\P\left(\bigcap_{s=0}^K \{m_{t-s-1} \not\in M^* \} \mid E \right) \le 1-\varepsilon^{N^4}. 
		\end{equation*}
		\item We next generalize the result to $K > N^4$. This is straightforward, as the random seeds $x$ in 
		Algorithm~\ref{alg:ucb_ca_random} are mutually independent for any non-overlapping blocks of $\numagents^4$ steps.
	\end{enumerate}
\end{proof}
Note that in order for this bound to be meaningful, we require $K \gg \varepsilon^{-N^4}N^4$.

Finally, we are now fully equipped to prove the main result of this section.

\begin{proof}\textbf{of Theorem~\ref{thm:regret_general_pref}}
	~~Let $0 \le h_t < t$ be a time window that we are free to choose in a way that depends on the time $t$. 
	By Lemma~\ref{lem:unstable-event-sets} and the union bound, we have 
	\begin{equation*}
		\Pr\left(\match_t \not\in M^* \right) 
		\le  \P\left(\bigcap_{s=0}^{h_t} (E_{t-s} \cap \left\{{m_{t-s-1} \notin M^*}\right\})\right) + \sum_{s=0}^{h_t} \P(E_{t-s}^c).
	\end{equation*}
	Let $g_t =\lfloor h_t/N^4\rfloor$. Lemmas \ref{lemma:bound-E-t-c} and \ref{lem:exponential-prob} immediately yield the
	following: 
	\begin{equation*}
		\Pr\left(\match_t \not\in M^* \right) 
		\le (1 - \varepsilon^{N^4})^{g_t} + \varepsilon^{-1} \sum_{s = 0}^{h_t} \sum_{(i, j, k), : \arm{j} \prec_i \arm{k}} \P(F_{j, k}^{(i)}(t - s)\cap  \pull{i}(t - s) = j ).
	\end{equation*}
	Summing these inequalities over $t$ up to $T$, we obtain
	\begin{equation}
		\begin{split}
			\sum_{t = 1}^{T} \Pr\left(\match_t \not\in M^* \right) 
			&\leq \sum_{t = 1}^{T} (1 - \varepsilon^{N^4})^{g_t} + 
			\varepsilon^{-1} \sum_{t = 1}^{T}\sum_{s = 0}^{h_t} \sum_{(i, j, k), : \arm{j} \prec_i \arm{k}} 
			\!\!\!\! \P(F_{j, k}^{(i)}(t - s)\cap  \pull{i}(t - s) = j ) \\
			&= \sum_{t = 1}^{T} (1 - \varepsilon^{N^4})^{g_t} + 
			\varepsilon^{-1} \sum_{(i, j, k), : \arm{j} \prec_i \arm{k}}   \sum_{s = 0}^{h_T} \sum_{\substack{t: s \leq h_t \\ 1 \le t\le T}} 
			\P(F_{j, k}^{(i)}(t - s)\cap  \pull{i}(t - s) = j ) 
			\label{eqn:master-equation}
		\end{split}
	\end{equation}
	We seek upper bounds for the terms on the right-hand side. Focus on the second term in equation~\eqref{eqn:master-equation}. 
	Recall the standard UCB Lemma (e.g., Lemma~\ref{lem:ucb_bound}):
	\begin{equation*}
		\sum_{\substack{t: s \leq h_t \\ 1 \le t\le T}} 
		\P(F_{j, k}^{(i)}(t - s)\cap  \pull{i}(t - s) = j) 
		\le 6 \cdot \left(\frac{1}{\Delta^2}\log (T) + 1\right), ~\text{for each $s$.}
	\end{equation*}
	
	Substituting this bound into equation~\eqref{eqn:master-equation} yields
	\begin{align}
		\sum_{t = 1}^{T } \Pr\left(\match_t \not\in M^* \right) \leq \sum_{t = 1}^{T} (1 - \varepsilon^{N^4})^{g_t}
		+ 6 \varepsilon^{-1} \numagents \numarms^2 
		\left(h_T + 1\right) \left(\frac{1}{\Delta^2}\log (T) + 1\right), 
		\label{eqn:master-equation-second}
	\end{align}
	where we have used the fact that there are at most $\numagents \numarms^2$ triplets $(i,j,k)$ such that 
	$\arm{j} \prec_i \arm{k}$. We now choose a specific sequence $(h_t)$ to optimize the upper bound. 
	Let $B \ge 1$ be determined later. Set $h_t = \min\{t, B\} - 1$.  With this choice of $h_t$, and after
	some elementary computations, we can bound the 
	first term in equation~\eqref{eqn:master-equation-second} by 
	\begin{equation*}
		\sum_{t = 1}^{T} (1 - \varepsilon^{N^4})^{g_t} \le 3 \cdot \sum_{t=1}^T \exp(-h_t \varepsilon^{N^4}/N^4)
		\le 6 \cdot \left(T \exp \left(- \frac{B \varepsilon^{N^4}}{2N^4}\right) + \frac{\numagents^4}{\varepsilon^{\numagents^4}}\right).
	\end{equation*}
	The second term in equation~\eqref{eqn:master-equation-second} is bounded by 
	$6 \varepsilon^{-1} B\numagents \numarms^2 
	\left(\frac{1}{\Delta^2}\log (T) + 1\right)$, since $h_T < B$ by definition. Consequently, these 
	two bounds lead to the following (that holds for all $B$)
	\begin{align*}
		\sum_{t = 1}^{T } \Pr\left(\match_t \not\in M^* \right) \leq& ~6 \cdot \left(T \exp \left(- \frac{B\varepsilon^{N^4}}{2N^4}\right) 
		+ \frac{\numagents^4}{\varepsilon^{\numagents^4}} + \frac{1}{\varepsilon} \numagents \numarms^2
		B\left(\frac{1}{\Delta^2}\log (T) + 1\right)\right).
	\end{align*}
	By carefully setting $B = 2\left\lceil \frac{\numagents^4}{\varepsilon^{\numagents^4}} 
	\log\left(T\right) \right\rceil$, we obtain the final bound as desired 
	\begin{equation*}
		\sum_{t = 1}^{T } \Pr\left(\match_t \not\in M^* \right) \leq 24 \cdot 
		\frac{\numagents^5 \numarms^2}{\varepsilon^{\numagents^4 + 1}} \log(T) \cdot \left(\frac{1}{\Delta^2}\log (T) + 3\right).
	\end{equation*}
\end{proof}

% !TeX root = main.tex 

\section{Strategy and Incentive Compatibility}
\label{sec:ic}

In this section, we examine the CA-UCB algorithm from the perspective of incentive compatibility.

Thus far we have given stable regret guarantees for each player, when all players follow the same algorithm, whether assuming a global ranking of players (Theorem~\ref{thm:global-players}), or without making assumptions on the market's preferences (Theorem~\ref{thm:regret_general_pref}). Given these results, a natural question to consider, in the decentralized setting, is whether the players are indeed incentivized to run the same algorithm as everyone else. In other words, could any single player benefit from running a different algorithm, when all other players are running Algorithm~\ref{alg:ucb_ca_random}? 

\subsection{A positive result for globally ranked players}

In the setting of Section~\ref{sec:global-players}, when players are globally ranked, we can show that the gains from deviating are limited. The following proposition gives an lower bound on the stable regret of the deviating player that scales logarithmically in the horizon $\horizon$, for any algorithm that they run. This implies that the time-averaged gains from deviating must vanish quickly as learning progresses.

\begin{proposition}[Incentive compatibility under globally ranked players]\label{prop:incentive-global}
	Under Assumption~\ref{assmp:global_players}, suppose that all players other than player $\player{k}$ run Algorithm~\ref{alg:ucb_ca_random} with $\stayprob = 0$, and $\player{k}$ can run any algorithm. The following lower bound on player $\player{k}$'s stable regret holds:
	\begin{equation}
	\regret{k}(\horizon) \ge 6 k^2 (\numarms-k) \left( \frac{\log \horizon}{\Delta^2} + 1\right)\left(\min_{j: \gap{k}{j} < 0}\gap{k}{j}\right).
	\end{equation}
\end{proposition}
This result follows from a simple application of the same arguments that we developed to prove Theorem~\ref{thm:global-players}. The key idea is as follows. A deviating player that is rank $k$ in the market can successfully pull an arm $\arm{i}$ that they prefer to their stable arm, only if the better-ranked player $\player{i}$ is not pulling their stable arm $\arm{i}$ in the same round. This can only happen if $\player{i}$ or a better-ranked player had a mistake in their UCB rankings and pulled a suboptimal arm within the last $k$ rounds, since all players other than $\player{k}$ are indeed following the CA-UCB algorithm. The gains to deviating are limited for player $\player{k}$ when all the arms have the same preferences, precisely because $\player{k}$ cannot affect the actions of better ranked players. A complete proof can be found in Appendix~\ref{app:global-IC-proof}.

\subsection{A negative result}
Given that we have a general stable regret guarantee for arbitrary preferences, established in Section~\ref{sec:general-prefs}, one might ask if there also exists a general incentive compatibility result for Algorithm~\ref{alg:ucb_ca_random}. Unfortunately, the answer is a negative one. The following proposition shows, by way of counterexample, that there can be no blanket incentive compatibility guarantee for Algorithm~\ref{alg:ucb_ca_random} without making additional assumptions, such as on the preference structure. 

%\todo{More elaboration. Compare with section 4's IC result. Also comment on future work directions, or add this to the discussion section.}

%\todo{Is it also worth commenting on how the Alg is challenging to analyze for incentive compatibility, because we cannot re-use the arguments to prove regret bounds to prove IC. This is because our proof of convergence assumes all players are following the Algorithm, and any deviating player will break several lemmas. This is unlike in Section 4, where the lemmas were player-wise, and one deviating player didn't affect better ranked players.}

\begin{proposition}\label{prop:not-incentive-comp}
	Consider the market of three players and three arms with preferences as given in Example~\ref{ex:incentive-counterex}. When two players $\player{1}$ and $\player{2}$ run Algorithm~\ref{alg:ucb_ca_random} with any $\stayprob \in (0,1/4)$, there exists a sequence of actions $\{\attempt{3}(t)\}_{t=1...\horizon}$ for player $\player{3}$ such that $\player{3}$'s stable regret can be upper bounded as:
	\begin{equation}
	\regret{3}(\horizon) \le -C_{1}\cdot \horizon + C_2\left(\frac{1}{\Delta^2}\log (T) + 1\right),
	\end{equation}
	where $C_1$ and $C_{2}$ are constants that depend only on $\lambda,\gap{3}{1}, \gap{3}{\emptyset}$. Moreover, there exists $\gap{3}{1}, \gap{3}{\emptyset}$ such that $C_{1}$ is strictly positive.
\end{proposition}

The above upper bound on the deviating player $\player{3}$'s stable regret shows that there exists a set of preferences and arm reward gaps such that a player could make significant gains over their stable arm by not running Algorithm~\ref{alg:ucb_ca_random}.  We defer the full description of Example~\ref{ex:incentive-counterex} and the proof of Proposition~\ref{prop:not-incentive-comp} to Appendix~\ref{app:general-ic}. In this example, $\player{3}$ has stable arm $\arm{3}$ but prefers $\arm{1}$. Because the arms have idiosyncratic preferences (as opposed to shared preferences), $\player{3}$ could pull a suboptimal arm in order to `trick' $\player{1}$ into not attempting $\arm{1}$ two rounds later, by exploiting the conflict avoidance mechanism; $\player{3}$ can then successfully pull $\arm{1}$ for one round, with some probability. As long as the reward for $\player{3}$ from $\arm{1}$ is large enough, $\player{3}$ is guaranteed a strictly negative stable regret that is linear in the horizon $\horizon$.

We have shown that Algorithm~\ref{alg:ucb_ca_random} is not incentive compatible in the fully general setting. It therefore remains an open question whether there exists an algorithm with low stable regret, under arbitrary preferences, that also has an incentive compatibility guarantee under the same.

% !TeX root = main.tex 

\section{Simulation experiments for random preferences}
\label{sec:expt}

In our theoretical analysis we considered two cases: markets in which the players are globally ranked (i.e. all arms have the same preferences over players) and markets with arbitrary preferences. For the first case Theorem~\ref{thm:global} we were able to prove a regret upper bound that resembles the guarantee derived by \cite{liu20competing} in the centralized case. However, in the case of general markets our guarantee (Theorem~\ref{thm:regret_general_pref}) has an exponential dependence on the size of the market. 

In this section, through empirical evaluations we show that 
the true performance of our proposed method is likely better than our guarantee suggests for markets with randomly drawn preferences. More precisely, we perform two sets of simulations. In the first set, we investigate how the average regret and market stability depend on the size of the market in balanced markets---markets with an equal number of players and arms---with preferences drawn from a distribution that will be specified later. We find that empirically the algorithm converges more slowly for larger number of players as expected, though the dependence on the number of players, $\numagents$, appears to be significantly better than the exponential dependence appearing in Theorem~\ref{thm:regret_general_pref}.

In the second set of experiments, we vary the heterogeneity of the players' preferences. We perform this experiment because one might expect that in markets in which different players have the same preferences there would be more conflicts (since different players have an incentive to attempt the same arms). Despite this intuition, our simulations show that CA-UCB performs equally well in markets with different level of heterogeneity. To sum up, our simulations show that not only is Theorem~\ref{thm:regret_general_pref} overly pessimistic, but that CA-UCB avoids conflicts equally well in different markets. 

For all experiments we use Algorithm~\ref{alg:ucb_ca_random} with delay probability $\stayprob=0.1$. We now present the details of our simulations.

\paragraph{Varying the size of the market.} We examine balanced markets of size $N \in \{5, 10, 15, 20\}$, and sample each player's and arm's ordinal preferences uniformly at random.  For all players the reward gaps between consecutively ranked arms are chosen to be equal to $\mingap=1$, regardless of the market size. The rewards are normally distributed with unit variance. \lledit{We sampled ten markets as such, and run Algorithm~\ref{alg:ucb_ca_random} once on each market.}

For each market size $\numagents$, we plot the mean, over ten markets, of the following two quantities: (i) the maximum average regret among players, $\max_{k\in\players}\regret{k}(\horizon)$, and (ii) the averaged market stability $\sum_{t = 1}^{T} \Pr\left(\match_t \not\in M^* \right) $ for horizon $\horizon$ up to $5000$. As can be seen in Figure~\ref{fig:players}, both the average regret and the market stability converge more slowly for larger markets. However, the dependence on $\numagents$ appears to be much better than exponential.

\paragraph{Varying the heterogeneity of the players' preferences.} We examine balanced markets of size $10$, and sample each arm's ordinal preferences uniformly at random. To sample the mean rewards $\meanreward{k}{i}$ of arm $\arm{i}$ for player $\player{k}$ we rely on random utility model used by \citet{ashlagi2017unbalanced}, with a slight modification:
\begin{align*}
x_i &\stackrel{i.i.d.}{\sim} \text{Uniform}([0,1]) \\
\varepsilon_{i,k} &\stackrel{i.i.d.}{\sim}  \text{Logistic}(0,1)\\
\overline{\mu}_i^{(k)} &= \beta x_i + \varepsilon_{i,k}\\
\meanreward{k}{i} &= \#\{j:\overline{\mu}_j^{(k)}\leq \overline{\mu}_i^{(k)} \}
\end{align*}
The intermediate utilities $\overline{\mu}_i^{(k)}$ are sampled according to random utility model used by \citet{ashlagi2017unbalanced}. We map these random utilities to $\meanreward{k}{i}$ so that the reward gaps between consecutively ranked arms are kept constant at $\mingap=1$. The parameter $\beta > 0$ determines the degree of correlation between the players' preferences. 
As $\beta$ increases the correlation between the players' preferences also increases. In fact, in the limit as $\beta \to \infty$, all the players share the same preferences with probability 1.

As before, the rewards are normally distributed with unit variance. We sample ten markets for each $\beta$ value, and plot the maximum average regret among players as well as the averaged market stability for horizon $\horizon$ up to $5000$ in Figure~\ref{fig:hetero}. As can be seen, there is no discernible difference in the convergence of Algorithm~\ref{alg:ucb_ca_random} in terms of regret or market stability, for markets with different levels of preference heterogeneity.

\begin{figure}[th]
	\centering
	\includegraphics[width=.5\textwidth]{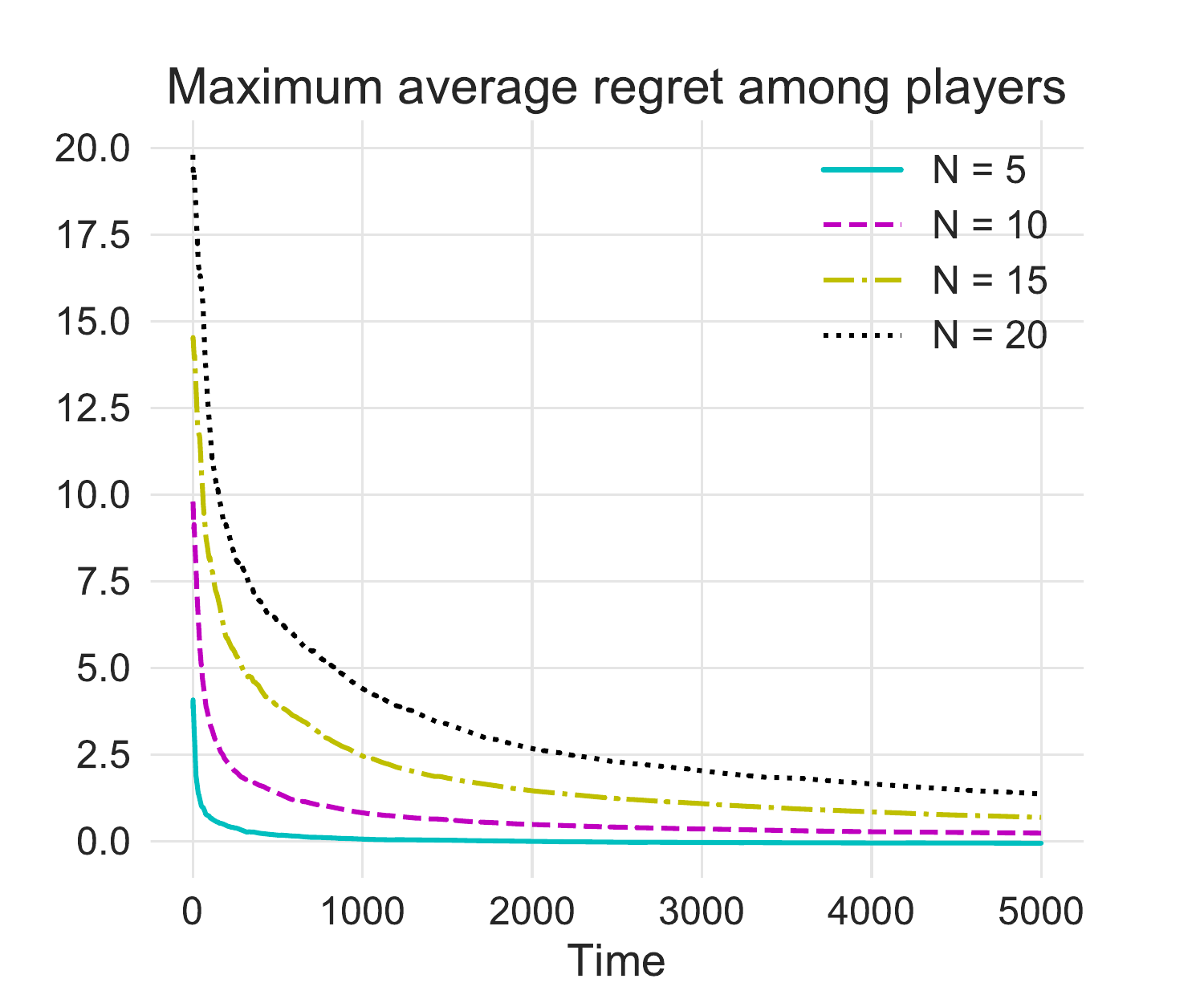}%
	\includegraphics[width=.5\textwidth]{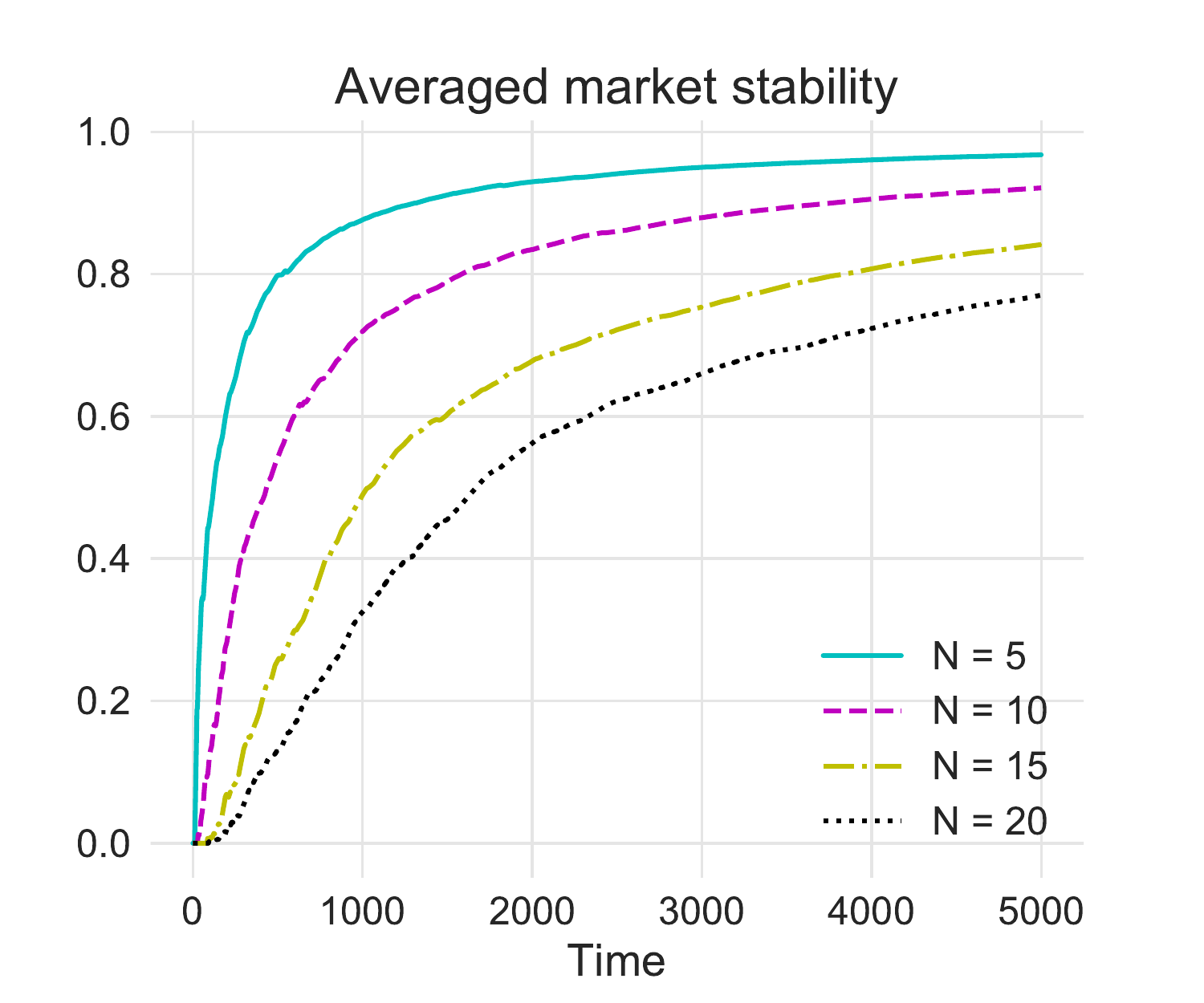}
	\caption{Varying the number of players. The plot on the left shows the maximum average regret among players and the plot on the right shows the averaged market stability. }\label{fig:players}
\end{figure}

\begin{figure}[th]
	\centering
	\includegraphics[width=.5\textwidth]{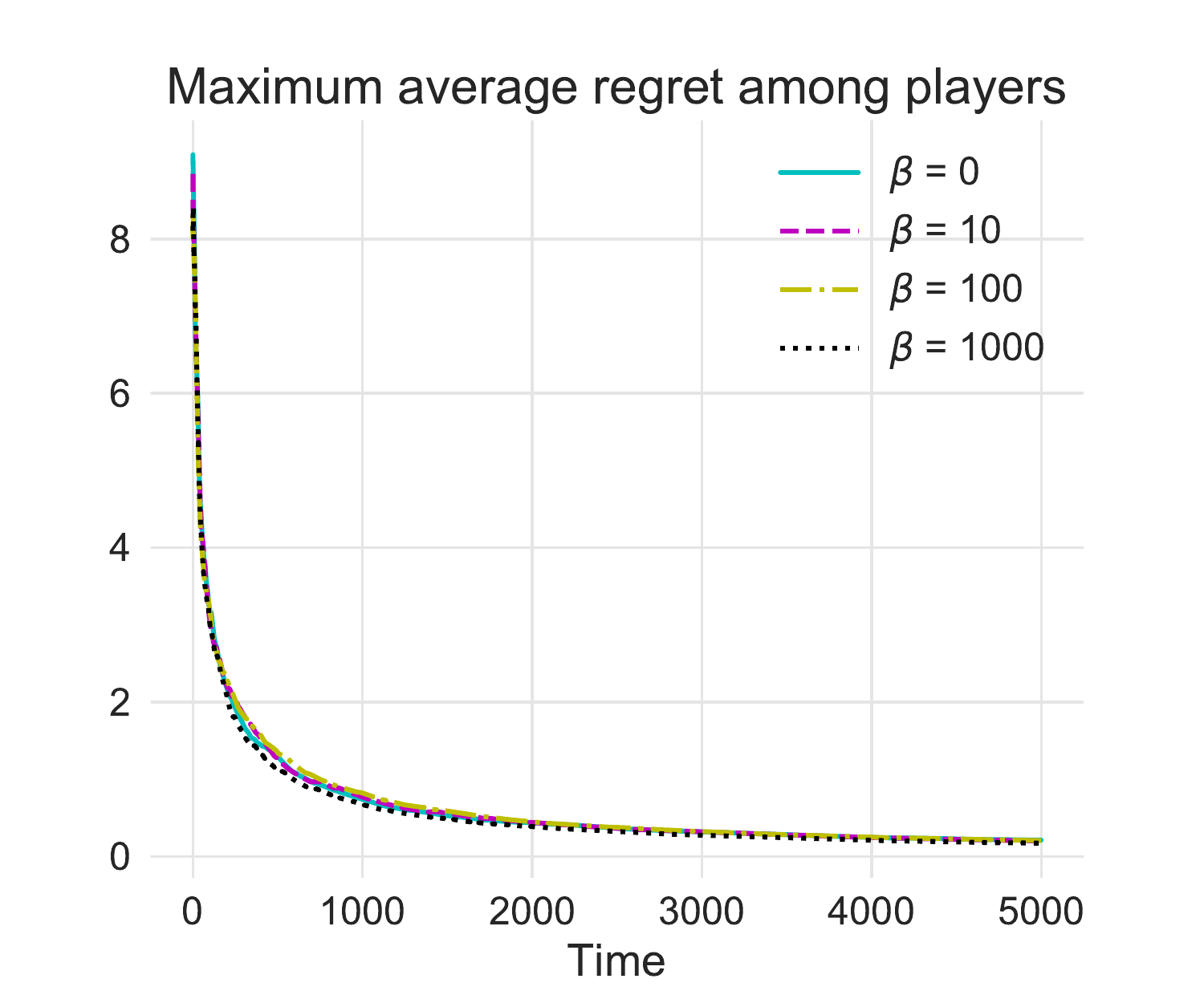}%
	\includegraphics[width=.5\textwidth]{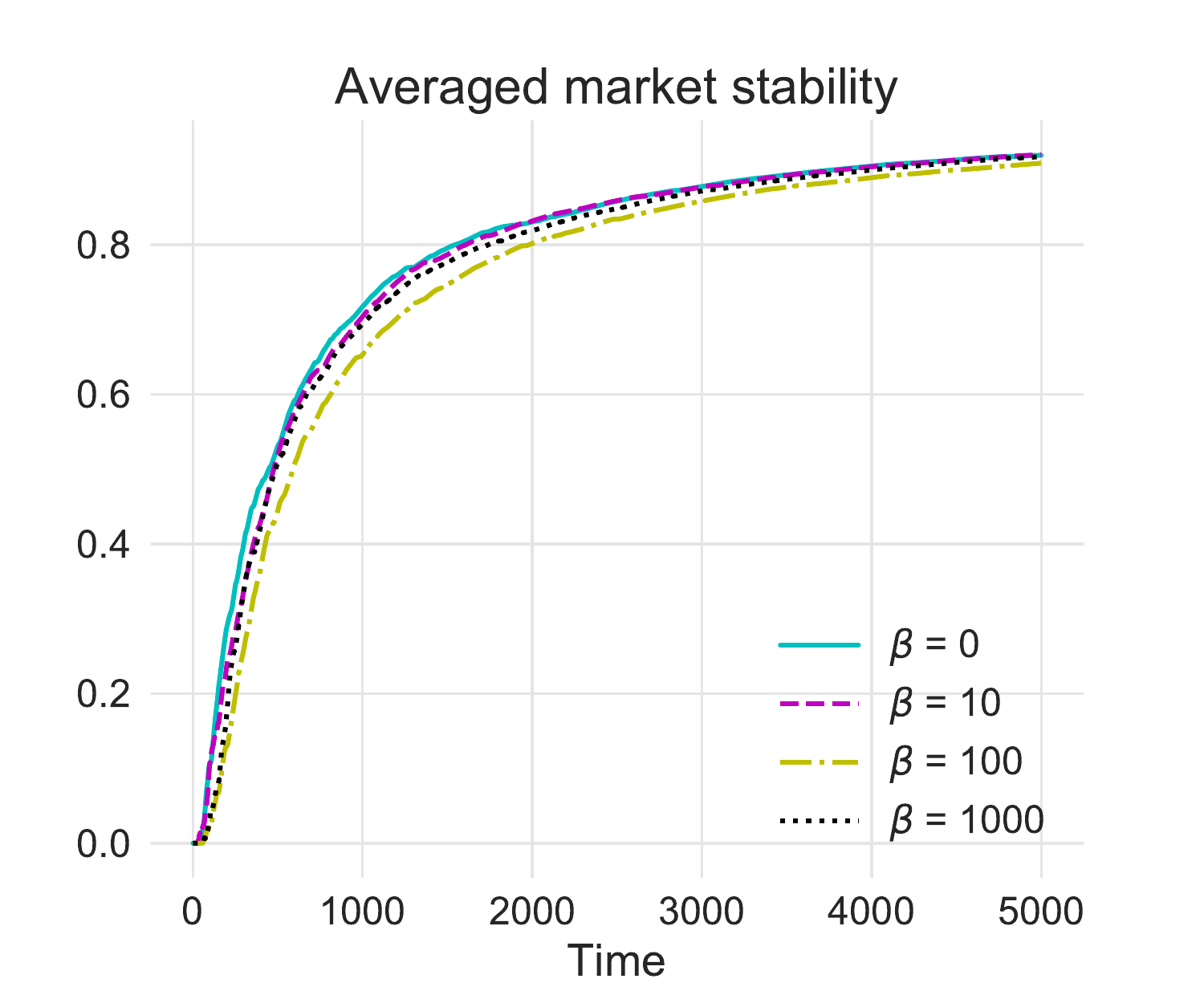}
	\caption{Varying the heterogeneity of the players' preferences. The plot on the left shows the maximum average regret among players and the plot on the right shows the averaged market stability. The larger the $\beta$ parameter, the more correlated the players' preferences are on average.}\label{fig:hetero}
\end{figure}

\section{Related Work}\label{sec:relwork}

There has been significant recent interest in stochastic multi-armed bandits problems with multiple, interacting 
players \citep{pmlr-v49-cesa-bianchi16,Shahrampour17}. In one 
formulation, known as \emph{bandits with collision}, multiple players choose from the same 
set of arms, and if two or more players choose the same arm, no reward is received by any 
player~\citep{Liu10distributed, Anandkumar:2011distri, avner2014, Bistritz2020MyFB, bubeck2020cooperative, bubeck2020non, kalathil2014dencentralized, rosenski16,lugosi18multi}.  In this setting, players are typically assumed to be cooperative, that is, their goal is to maximize the collective reward. \citet{Bistriz2018} and \citet{boursier20selfish} consider the setting where agents have heterogeneous preferences over arms, and the latter work also analyzes the effect of selfish players whose goal is to maximize individual rewards. \citet{avner2016multi} and \citet{darak2019multi} considered a ``stable configuration'' as a solution concept in the heterogeneous player preference setting; however, because the arms do not have preferences in their setting, their notion of ``stability'' is distinct from that of two-sided stable matching. \citet{bubeck2020non} also delineated
the optimal rates for the non-stochastic version of the cooperative problem.

\citet{liu20competing} introduced a multi-player stochastic multi-armed bandits problem motivated by two-sided matching markets, where arms also
have preferences, and in case of collision only the most preferred player receives a reward.
Unlike in the aforementioned line of work, where the natural goal is to find a maximum matching between players and arms, a more appropriate goal here is to find a stable matching. In the centralized setting, where a platform can coordinate the actions of players at each round, \citet{liu20competing}'s algorithm combining the upper confidence bound method and the deferred acceptance algorithm attains $\Ocal(\log(T)/\mingap^2)$ stable regret, which is order-optimal. A suboptimal algorithm based on explore-then-commit was proposed for the decentralized setting. Follow-up work by \citet{Sankararaman20dominate} on the decentralized setting analyzed an order-optimal algorithm for globally ranked players. A more detailed discussion of this work is in Section~\ref{sec:dis}.

The two-sided stable matching problem with preference learning has been studied in other dynamic settings under different assumptions. Given the large space of modeling choices, there has been a flowering of research on two-sided matching models that highlight different challenges introduced by uncertainty and decentralization. One modeling choice is to define arrival and departure processes for market participants, as opposed to analyzing a fixed set of players and arms. \citet{Johari2017matching} studied a sequential matching problem in which the market participants satisfy certain arrival processes, and the participants on the demand side of the market have a `type' that is learned through bandit feedback.

Another choice is how one formulates the cost of preference learning. \citet{Ashlagi2017communication}, which studies the costs of communication and learning for stable matching, formulates preference learning as querying a costly but noiseless choice function. Different players can query their choice functions independently; thus there is no congestion in the preference learning process. Many models studied in the literature on information acquisition in two sided matching \citep[see][and references therein]{lee2009interviewing,immorlica2020information} also do not capture congestion in the information acquisition stage. In some markets, however, obtaining information about the other side of the market itself could lead to congestion and thus the need for strategic decisions. For example, \citet[][chap. 10]{roth1990two} note that graduating medical students go to interviews to ascertain their own preferences for hospitals, but the collection of interviews that a student can schedule is limited. In the model that is studied in the current work, congestion in preference learning is captured by conflicts when two or more players attempt to pull the same arm. %Our model begins to capture such tradeoffs by introducing statistical uncertainty in the preferences of one side of the market and providing a natural mode of interaction between the learning agents.

Other models of uncertainty in two-sided matching that do not explicitly consider preference learning have also been studied. In this setting, there has been much interest in decentralized models. For example, \citet{niederle2009decentralized} studied a decentralized market game in which firms make directed offers to workers, agents have aligned preferences, and equilibrium outcomes under preference uncertainty are analyzed. \citet{arnosti2014managing} employed mean field modeling to analyze the welfare costs of not knowing the availability of agents, as opposed to preferences. \citet{ashlagi2019assortment} considered providing match recommendations to participants in markets for which both sides of the market propose with some probability, and a successful match occurs only in the case of a mutual proposal. \citet{dai2020learning} study a single-stage matching problem with uncertain preferences where players learn from historical data and act in a decentralized manner.

Lastly, the empirical aspects of stable matching in decentralized settings have also garnered significant research interest \citep{Das2005two,echenique2012experimental,pais2012decentralized}.
%performed an empirical study of a two-sided matching bandits problem where both sides of the market have uncertain preferences.

\section{Discussion}\label{sec:dis}

In this section, we discuss the strengths and limitations of Algorithm~\ref{alg:ucb_ca_random}, in the context of broader themes in decentralized matching and multiplayer bandit learning. We also suggest future research directions motivated by our current findings.

\paragraph{Single-phase algorithm} One advantage of Algorithm~\ref{alg:ucb_ca_random} is its simplicity, specifically the fact that it does not involve separate phases or subroutines. Recent work by \citet{Sankararaman20dominate} studied an algorithm (`UCB-D3') for decentralized matching bandits, assuming \emph{globally ranked players}, that proceeds in phases of exponentially increasing length; each phase comprises of a learning stage, where players choose arms according to their own UCBs, followed by a communication subroutine, where players broadcast their preferred arms to other players. In contrast, our algorithm does not require players to keep track of which phase they are in, or when to begin a subroutine. Not having separate algorithmic phases is desirable because multiple phases requires players to synchronize their transition from one phase to the next. In `more decentralized' situations this may not be possible. 
For example, players may enter the market at different times, or leave the market for a number of rounds only to return later \citep[see e.g.,][]{Akbarpour20thick}. The CA-UCB algorithm can be run in such cases without modification and is still guaranteed to have small regret.

\paragraph{Dependence of stable regret on market size} While both UCB-D3 and our method are guaranteed to achieve $\Ocal(\log(T)/\mingap^2)$ stable regret for globally ranked players, the regret guarantee for UCB-D3 has a better dependence on the number of arms (which upper bounds the number of players). In the worst case, the guarantee on the regret of UCB-D3 depends on the square of the number of arms while the guarantee on the regret of our method depends on the cube of the number of arms. The optimal order-dependence on the rank $k$ and the number of arms $\numarms$ is still an open question, since the lower bound \citep[e.g., Corollary 6 in][]{Sankararaman20dominate} and upper bounds currently do not match.  Another interesting question is whether UCB-D3's better regret guarantee under these assumptions translates to better performance in practice; an in-depth empirical comparison of UCB-D3 and CA-UCB will be needed and is beyond the scope of the current work.

\paragraph{Random delays} Another important feature of Algorithm~\ref{alg:ucb_ca_random} is the injection of additional randomness through each player's independently drawn random delays. Randomization is key for this algorithm to achieve a $\Ocal(\log(T)^2)$ regret guarantee in the case of \emph{arbitrary} two-sided preferences. Intuitively, the added randomness allows players to escape conflict cycles, as illustrated in Examples~\ref{ex:2player} and \ref{ex:3player}. Technically, it allows us to leverage a result from \citet{hernan95paths}) to show that the players must converge to a stable matching (which may not be unique), in a low-regret sense. Nevertheless, repeating one's previous action with a constant probability at every step could be considered wasteful. Are there other, more efficient ways of utilizing randomness as an implicit coordination mechanism than random delays?

\paragraph{Improving the stable regret under arbitrary preferences} While Algorithm~\ref{alg:ucb_ca_random} is the first method to provably achieve polylogarithmic regret in markets with arbitrary preferences, we believe there is a significant room for the development of better algorithms. In particular, for markets with arbitrary preferences, the regret guarantee for our method depends exponentially on the number of players. This dependence arises because our regret analysis hinges on a reduction to the convergence rate of the corresponding randomized decentralized matching dynamics under \emph{known} preferences. As shown in \citet{ackerman2008uncoord} and \citet{hoffman2013jealousy}, existing randomized dynamics for decentralized matching under known preferences have worst-case convergence time that is exponential in the number of market participants. While this may suggest that there is indeed a real computational barrier in the arbitrary preferences setting, it might be possible to improve upon the exponential dependence by considering sub-classes of two-sided preferences or randomly drawn preferences. For example, Algorithm~\ref{alg:ucb_ca_random} has improved rates if we assume that the players are globally ranked. %A global ranking (i.e., arms with totally shared preferences) is arguably restrictive as an assumption and may not hold in many markets.

It is also not clear that the $\Ocal((\log \horizon)^2)$ dependence on the horizon is optimal in this setting, even though it is unavoidable given our analysis strategy and our algorithm. Obtaining a regret bound that depends polynomially on the number of players and arms \emph{and} has an optimal order dependence on the horizon may require a new algorithm. 

\paragraph{Information available to players} \lledit{A player that implements CA-UCB must observe the successful arm pulls of all other players. On one hand, by leveraging this information our algorithm ensures that players avoid conflicts most of the time. On the other hand, it is not clear that such information is absolutely necessary for achieving sublinear regret in general markets. For example, UCB-D3 \citep{Sankararaman20dominate}, which achieves sublinear regret in the setting of globally ranked players, does not require players to see the actions of other players. However, players must participate in a rank estimation routine, which relies on the assumption that the players are ranked globally.}

\paragraph{Conclusion and open questions} In this work we have made progress on the problem of stochastic bandits in decentralized matching markets. Still, many open questions remain. %More work is needed to better understand of the interplay between the statistical learning and the economic incentives aspects of matching in decentralized stochastic environments. 
We conclude by highlighting the most intriguing directions for future inquiry: 
\begin{enumerate}
	\item \emph{Better algorithms and matching lower bounds.} Even though algorithms such as UCB-D3 \citep{Sankararaman20dominate} and CA-UCB have stable regret that is almost order-optimal in the setting of globally ranked players, there is still a lot of room for improvement in the setting of arbitrary preferences. Is there a large class of preferences for which one can show matching upper and lower regret bounds, in terms of the dependence on the horizon, the reward gap, and size of the market?
	\item \emph{Incentive compatibility in the decentralized setting.} Unlike in the centralized setting, where a single algorithm was shown to be incentive compatible given any set of preferences \citep{liu20competing}, decentralization appears to pose more challenges for incentive compatibility. As seen in Section~\ref{sec:ic}, the randomized conflict avoidance mechanism of Algorithm~\ref{alg:ucb_ca_random} can be strategically exploited by a deviating player when arm preferences are uncorrelated. How fundamental is this difficulty to the decentralized setting, and can it be overcome by a better algorithm?
\end{enumerate}

\bibliographystyle{abbrvnat}
\bibliography{recmkt,recmkt_2}

\begin{thebibliography}{47}
\providecommand{\natexlab}[1]{#1}
\providecommand{\url}[1]{\texttt{#1}}
\expandafter\ifx\csname urlstyle\endcsname\relax
  \providecommand{\doi}[1]{doi: #1}\else
  \providecommand{\doi}{doi: \begingroup \urlstyle{rm}\Url}\fi

\bibitem[Abdulkadiro{\u g}lu and S{\"o}nmez(2003)]{abdulkadirouglu2003school}
A.~Abdulkadiro{\u g}lu and T.~S{\"o}nmez.
\newblock School choice: A mechanism design approach.
\newblock \emph{American economic review}, 93\penalty0 (3):\penalty0 729--747,
  2003.

\bibitem[Abdulkadiro{\u g}lu et~al.(2006)Abdulkadiro{\u g}lu, Pathak, Roth, and
  Sonmez]{abdulkadiroglu2006changing}
A.~Abdulkadiro{\u g}lu, P.~Pathak, A.~E. Roth, and T.~Sonmez.
\newblock Changing the boston school choice mechanism.
\newblock Technical report, National Bureau of Economic Research, 2006.

\bibitem[Abeledo and Rothblum(1995)]{hernan95paths}
H.~Abeledo and U.~G. Rothblum.
\newblock Paths to marriage stability.
\newblock \emph{Discrete Applied Mathematics}, 63:\penalty0 1--12, 10 1995.

\bibitem[Ackermann et~al.(2008)Ackermann, Goldberg, Mirrokni, R\"{o}glin, and
  V\"{o}cking]{ackerman2008uncoord}
H.~Ackermann, P.~W. Goldberg, V.~S. Mirrokni, H.~R\"{o}glin, and
  B.~V\"{o}cking.
\newblock Uncoordinated two-sided matching markets.
\newblock In \emph{Proceedings of the 9th ACM Conference on Electronic
  Commerce}, pages 256--263, 2008.

\bibitem[Akbarpour et~al.(2020)Akbarpour, Li, and Gharan]{Akbarpour20thick}
M.~Akbarpour, S.~Li, and S.~O. Gharan.
\newblock Thickness and information in dynamic matching markets.
\newblock \emph{Journal of Political Economy}, 128\penalty0 (3):\penalty0
  783--815, 2020.

\bibitem[Anandkumar et~al.(2011)Anandkumar, Michael, Tang, and
  Swami]{Anandkumar:2011distri}
A.~Anandkumar, N.~Michael, A.~K. Tang, and A.~Swami.
\newblock Distributed algorithms for learning and cognitive medium access with
  logarithmic regret.
\newblock \emph{IEEE Journal on Selected Areas in Communications}, 29\penalty0
  (4):\penalty0 731--745, 2011.

\bibitem[Aridor et~al.(2019)Aridor, Liu, Slivkins, and Wu]{Aridor2019competing}
G.~Aridor, K.~Liu, A.~Slivkins, and Z.~S. Wu.
\newblock Competing bandits: The perils of exploration under competition.
\newblock \emph{The 20th ACM Conference on Economics and Computation}, 2019.

\bibitem[Arnosti et~al.(2014)Arnosti, Johari, and Kanoria]{arnosti2014managing}
N.~Arnosti, R.~Johari, and Y.~Kanoria.
\newblock Managing congestion in decentralized matching markets.
\newblock In \emph{Proceedings of the fifteenth ACM conference on Economics and
  computation}, pages 451--451, 2014.

\bibitem[Ashlagi et~al.(2017{\natexlab{a}})Ashlagi, Braverman, Kanoria, and
  Shi]{Ashlagi2017communication}
I.~Ashlagi, M.~Braverman, Y.~Kanoria, and P.~Shi.
\newblock Communication requirements and informative signaling in matching
  markets.
\newblock In \emph{Proceedings of the 2017 ACM Conference on Economics and
  Computation}, EC '17, pages 263--263, 2017{\natexlab{a}}.

\bibitem[Ashlagi et~al.(2017{\natexlab{b}})Ashlagi, Kanoria, and
  Leshno]{ashlagi2017unbalanced}
I.~Ashlagi, Y.~Kanoria, and J.~D. Leshno.
\newblock Unbalanced random matching markets: The stark effect of competition.
\newblock \emph{Journal of Political Economy}, 125\penalty0 (1):\penalty0
  69--98, 2017{\natexlab{b}}.

\bibitem[Ashlagi et~al.(2019)Ashlagi, Krishnaswamy, Makhijani, Sab{\'{a}}n, and
  Shiragur]{ashlagi2019assortment}
I.~Ashlagi, A.~K. Krishnaswamy, R.~M. Makhijani, D.~Sab{\'{a}}n, and
  K.~Shiragur.
\newblock Assortment planning for two-sided sequential matching markets.
\newblock \emph{CoRR}, abs/1907.04485, 2019.

\bibitem[Avner and Mannor(2014)]{avner2014}
O.~Avner and S.~Mannor.
\newblock Concurrent bandits and cognitive radio networks.
\newblock In T.~Calders, F.~Esposito, E.~H{\"u}llermeier, and R.~Meo, editors,
  \emph{Machine Learning and Knowledge Discovery in Databases}, pages 66--81,
  2014.

\bibitem[{Avner} and {Mannor}(2016)]{avner2016multi}
O.~{Avner} and S.~{Mannor}.
\newblock Multi-user lax communications: A multi-armed bandit approach.
\newblock In \emph{The 35th Annual IEEE International Conference on Computer
  Communications}, pages 1--9, 2016.

\bibitem[Bistritz and Leshem(2018)]{Bistriz2018}
I.~Bistritz and A.~Leshem.
\newblock Distributed multi-player bandits---{A} game of thrones approach.
\newblock In S.~Bengio, H.~Wallach, H.~Larochelle, K.~Grauman, N.~Cesa-Bianchi,
  and R.~Garnett, editors, \emph{Advances in Neural Information Processing
  Systems 31}, pages 7222--7232, 2018.

\bibitem[Bistritz et~al.(2020)Bistritz, Baharav, Leshem, and
  Bambos]{Bistritz2020MyFB}
I.~Bistritz, T.~Z. Baharav, A.~Leshem, and N.~Bambos.
\newblock My fair bandit: Distributed learning of max-min fairness with
  multi-player bandits.
\newblock In \emph{Proceedings of The 37th International Conference on Machine
  Learning}, 2020.

\bibitem[Boursier and Perchet(2020)]{boursier20selfish}
E.~Boursier and V.~Perchet.
\newblock Selfish robustness and equilibria in multi-player bandits.
\newblock In J.~Abernethy and S.~Agarwal, editors, \emph{Proceedings of the
  33rd Conference on Learning Theory}, volume 125 of \emph{Proceedings of
  Machine Learning Research}, pages 530--581, 2020.

\bibitem[Bubeck and Cesa{-}Bianchi(2012)]{Bubeck12regret}
S.~Bubeck and N.~Cesa{-}Bianchi.
\newblock Regret analysis of stochastic and nonstochastic multi-armed bandit
  problems.
\newblock \emph{Foundations and Trends in Machine Learning}, 5\penalty0
  (1):\penalty0 1--122, 2012.

\bibitem[Bubeck et~al.(2020{\natexlab{a}})Bubeck, Budzinski, and
  Sellke]{bubeck2020cooperative}
S.~Bubeck, T.~Budzinski, and M.~Sellke.
\newblock Cooperative and stochastic multi-player multi-armed bandit: Optimal
  regret with neither communication nor collisions.
\newblock \emph{arXiv preprint arXiv:2011.03896}, 2020{\natexlab{a}}.

\bibitem[Bubeck et~al.(2020{\natexlab{b}})Bubeck, Li, Peres, and
  Sellke]{bubeck2020non}
S.~Bubeck, Y.~Li, Y.~Peres, and M.~Sellke.
\newblock Non-stochastic multi-player multi-armed bandits: Optimal rate with
  collision information, sublinear without.
\newblock In \emph{Proceedings of the 33rd Conference on Learning Theory},
  pages 961--987, 2020{\natexlab{b}}.

\bibitem[Cen and Shah(2021)]{cen2021regret}
S.~H. Cen and D.~Shah.
\newblock Regret, stability, and fairness in matching markets with bandit
  learners.
\newblock \emph{arXiv preprint arXiv:2102.06246}, 2021.

\bibitem[Cesa-Bianchi et~al.(2016)Cesa-Bianchi, Gentile, Mansour, and
  Minora]{pmlr-v49-cesa-bianchi16}
N.~Cesa-Bianchi, C.~Gentile, Y.~Mansour, and A.~Minora.
\newblock Delay and cooperation in nonstochastic bandits.
\newblock In V.~Feldman, A.~Rakhlin, and O.~Shamir, editors, \emph{29th Annual
  Conference on Learning Theory}, volume~49 of \emph{Proceedings of Machine
  Learning Research}, pages 605--622, 23--26 Jun 2016.

\bibitem[Dai and Jordan(2020)]{dai2020learning}
X.~Dai and M.~I. Jordan.
\newblock Learning strategies in decentralized matching markets under uncertain
  preferences.
\newblock \emph{arXiv preprint arXiv:2011.00159}, 2020.

\bibitem[{Darak} and {Hanawal}(2019)]{darak2019multi}
S.~J. {Darak} and M.~K. {Hanawal}.
\newblock Multi-player multi-armed bandits for stable allocation in
  heterogeneous ad-hoc networks.
\newblock \emph{IEEE Journal on Selected Areas in Communications}, 37\penalty0
  (10):\penalty0 2350--2363, 2019.

\bibitem[Das and Kamenica(2005)]{Das2005two}
S.~Das and E.~Kamenica.
\newblock Two-sided bandits and the dating market.
\newblock In \emph{Proceedings of the 19th International Joint Conference on
  Artificial Intelligence}, pages 947--952, 2005.

\bibitem[Echenique and Yariv(2012)]{echenique2012experimental}
F.~Echenique and L.~Yariv.
\newblock An experimental study of decentralized matching.
\newblock 2012.

\bibitem[Fudenberg and Levine(1998)]{fudenberg1998theory}
D.~Fudenberg and D.~K. Levine.
\newblock \emph{The theory of learning in games}, volume~2.
\newblock MIT press, 1998.

\bibitem[Gale and Shapley(1962)]{gale62college}
D.~Gale and L.~S. Shapley.
\newblock College admissions and the stability of marriage.
\newblock \emph{The American Mathematical Monthly}, 69\penalty0 (1):\penalty0
  9--15, 1962.

\bibitem[Hoffman et~al.(2013)Hoffman, Moeller, and Paturi]{hoffman2013jealousy}
M.~Hoffman, D.~Moeller, and R.~Paturi.
\newblock Jealousy graphs: Structure and complexity of decentralized stable
  matching.
\newblock In \emph{Web and Internet Economics}, pages 263--276, 2013.

\bibitem[Hu et~al.(1998)Hu, Wellman, et~al.]{hu1998multiagent}
J.~Hu, M.~P. Wellman, et~al.
\newblock Multiagent reinforcement learning: theoretical framework and an
  algorithm.
\newblock In \emph{ICML}, volume~98, pages 242--250. Citeseer, 1998.

\bibitem[Immorlica et~al.(2020)Immorlica, Leshno, Lo, and
  Lucier]{immorlica2020information}
N.~Immorlica, J.~Leshno, I.~Lo, and B.~Lucier.
\newblock Information acquisition in matching markets: The role of price
  discovery.
\newblock \emph{Available at SSRN}, 2020.

\bibitem[Johari et~al.(2017)Johari, Kamble, and Kanoria]{Johari2017matching}
R.~Johari, V.~Kamble, and Y.~Kanoria.
\newblock Matching while learning.
\newblock In \emph{ACM Conference on Economics and Computation}, pages
  119--119, 2017.

\bibitem[{Kalathil} et~al.(2014){Kalathil}, {Nayyar}, and
  {Jain}]{kalathil2014dencentralized}
D.~{Kalathil}, N.~{Nayyar}, and R.~{Jain}.
\newblock Decentralized learning for multiplayer multiarmed bandits.
\newblock \emph{IEEE Transactions on Information Theory}, 60\penalty0
  (4):\penalty0 2331--2345, 2014.

\bibitem[Knuth(1997)]{knuth97stable}
D.~E. Knuth.
\newblock \emph{Stable Marriage and its Relation to Other Combinatorial
  Problems}.
\newblock American Mathematical Society, 1997.

\bibitem[Lai and Robbins(1985)]{Lai85asymp}
T.~Lai and H.~Robbins.
\newblock Asymptotically efficient adaptive allocation rules.
\newblock \emph{Advances in Applied Mathematics}, 6\penalty0 (1):\penalty0 4 --
  22, 1985.

\bibitem[Lee and Schwarz(2009)]{lee2009interviewing}
R.~S. Lee and M.~Schwarz.
\newblock Interviewing in two-sided matching markets.
\newblock Technical report, National Bureau of Economic Research, 2009.

\bibitem[Littman(1994)]{littman1994markov}
M.~L. Littman.
\newblock Markov games as a framework for multi-agent reinforcement learning.
\newblock In \emph{Machine learning proceedings 1994}, pages 157--163.
  Elsevier, 1994.

\bibitem[Liu and Zhao(2010)]{Liu10distributed}
K.~Liu and Q.~Zhao.
\newblock Distributed learning in multi-armed bandit with multiple players.
\newblock \emph{IEEE Transactions on Signal Processing}, 58\penalty0
  (11):\penalty0 5667--5681, 2010.

\bibitem[Liu et~al.(2020)Liu, Mania, and Jordan]{liu20competing}
L.~T. Liu, H.~Mania, and M.~Jordan.
\newblock Competing bandits in matching markets.
\newblock In \emph{International Conference on Artificial Intelligence and
  Statistics}, volume 108, pages 1618--1628, 26--28 Aug 2020.

\bibitem[Lugosi and Mehrabian(2018)]{lugosi18multi}
G.~Lugosi and A.~Mehrabian.
\newblock Multiplayer bandits without observing collision information.
\newblock \emph{arXiv preprint arXiv:1808.08416}, 2018.

\bibitem[Mansour et~al.(2018)Mansour, Slivkins, and Wu]{MansourSW18}
Y.~Mansour, A.~Slivkins, and Z.~S. Wu.
\newblock Competing bandits: Learning under competition.
\newblock In \emph{9th Innovations in Theoretical Computer Science Conference,
  {ITCS} 2018, January 11-14, 2018, Cambridge, MA, {USA}}, pages 48:1--48:27,
  2018.

\bibitem[Niederle and Yariv(2009)]{niederle2009decentralized}
M.~Niederle and L.~Yariv.
\newblock Decentralized matching with aligned preferences.
\newblock Technical report, National Bureau of Economic Research, 2009.

\bibitem[Pais et~al.(2012)Pais, Pint{\'e}r, and Veszteg]{pais2012decentralized}
J.~Pais, A.~Pint{\'e}r, and R.~F. Veszteg.
\newblock Decentralized matching markets: a laboratory experiment.
\newblock 2012.

\bibitem[Rosenski et~al.(2016)Rosenski, Shamir, and Szlak]{rosenski16}
J.~Rosenski, O.~Shamir, and L.~Szlak.
\newblock Multi-player bandits---{A} musical chairs approach.
\newblock In \emph{Proceedings of The 33rd International Conference on Machine
  Learning}, volume~48, pages 155--163, 2016.

\bibitem[Roth and Sotomayor(1990)]{roth1990two}
A.~E. Roth and M.~A.~O. Sotomayor.
\newblock \emph{Two-Sided Matching: A Study in Game-Theoretic Modeling and
  Analysis}.
\newblock Econometric Society Monographs. Cambridge University Press, 1990.

\bibitem[Roth and {Vande Vate}(1990)]{roth90random}
A.~E. Roth and J.~H. {Vande Vate}.
\newblock Random paths to stability in two-sided matching.
\newblock \emph{Econometrica}, 58\penalty0 (6):\penalty0 1475--1480, 1990.

\bibitem[{Sankararaman} et~al.(2020){Sankararaman}, {Basu}, and {Abinav
  Sankararaman}]{Sankararaman20dominate}
A.~{Sankararaman}, S.~{Basu}, and K.~{Abinav Sankararaman}.
\newblock Dominate or delete: Decentralized competing bandits with uniform
  valuation.
\newblock \emph{arXiv preprint arXiv:2006.15166}, 2020.

\bibitem[{Shahrampour} et~al.(2017){Shahrampour}, {Rakhlin}, and
  {Jadbabaie}]{Shahrampour17}
S.~{Shahrampour}, A.~{Rakhlin}, and A.~{Jadbabaie}.
\newblock Multi-armed bandits in multi-agent networks.
\newblock In \emph{IEEE International Conference on Acoustics, Speech and
  Signal Processing}, pages 2786--2790, 2017.

\end{thebibliography}


\begin{thebibliography}{0}
\providecommand{\natexlab}[1]{#1}
\providecommand{\url}[1]{\texttt{#1}}
\expandafter\ifx\csname urlstyle\endcsname\relax
  \providecommand{\doi}[1]{doi: #1}\else
  \providecommand{\doi}{doi: \begingroup \urlstyle{rm}\Url}\fi

\end{thebibliography}

\newpage
\appendix

\section{Example~\ref{ex:3player}}\label{app:example}

In this section, we present a second counterexample in which CA-UCB without random delays (i.e.,~$\stayprob = 0$) would fail to converge to a stable matching and the players can enter into a conflict cycle. In this example, neither the arms nor the players are globally ranked. In contrast to Example~\ref{ex:2player}, the type of coordination failure seen in Example~\ref{ex:3player} is unrelated to the failure of the players to learn their rewards. In fact, they can enter into such a cycle even after they have acquired perfect information on all the arms.

 \begin{example}[3-player market with non-unique stable matching]\label{ex:3player}
 	\label{example:pessimal_regret}
 	Let the set of players be $\players = \{ \player{1}, \player{2}, \player{3} \}$ and the set of arms be $\arms = \{ a_1, a_2, a_3 \}$, with true preferences given by:
 	\begin{align*}
 		&\player{1}: \arm{3} \succ \arm{2} \succ \arm{1} &\quad \arm{1}: \player{3} \succ \player{2} \succ \player{1} \\
 		&\player{2}: \arm{1} \succ \arm{3}  \succ \arm{2} &\quad \arm{2}: \player{1} \succ \player{3} \succ \player{2} \\
 		&\player{3}: \arm{2} \succ \arm{1} \succ \arm{3} &\quad \arm{3}: \player{2} \succ \player{1} \succ \player{3}.
 	\end{align*}

 Then the conflict-avoiding algorithm cycles even when the preferences of the players are known. Suppose the players are following Algorithm~\ref{alg:ucb_ca_random}, and their UCB rankings for the arms always coincide with their true preferences. The cycle it enters is as follows:
 \begin{itemize}
 	\item Time $t$: $\player{1}$ and $\player{3}$ conflict on $\arm{2}$, $\player{1}$ wins. 

 	$\player{2}$ pulls $\arm{1}$.
 	\item Time $t+1$: $\player{3}$ attempts $\arm{1}$ because $\arm{2}$ is not in its plausible set. 	$\player{2}$ and $\player{3}$ conflict on $\arm{1}$, $\player{3}$ wins. 

 	$\player{1}$ pulls $\arm{3}$ because $\arm{3}$ was not pulled by any player at time $t$.
 	\item Time $t+2$: $\player{2}$ attempts $\arm{3}$ because $\arm{1}$ is not in its plausible set. $\player{1}$ and $\player{2}$ conflict on $\arm{3}$, $\player{2}$ wins.

 		$\player{3}$ pulls $\arm{2}$ because $\arm{2}$ was not pulled by any player at time $t+1$
 	\end{itemize}
 At time $t+3$, the players attempt the same actions as they did at time $t$, entering into a cycle where there is a conflict at every round henceforth.

 \end{example}

 Previous work has found other examples where sequentially resolving blocking pairs in an unstable matching leads to cycling \citep{knuth97stable,roth90random,hernan95paths}. Example~\ref{ex:3player} shows that players following the decentralized conflict-avoiding protocol (where more than one blocking pair may be resolved at every time step) can also enter into cycles.

 These examples highlight the failure modes of decentralized conflict-avoiding algorithms. One way to escape these failure modes is by introducing randomness, such that the probability of coordination failures becomes exponentially small. This is the motivation for incorporating random delays into Algorithm~\ref{alg:ucb_ca_random}.

\newpage

\section{Proof of Lemma~\ref{lem:ucb_bound}}\label{app:general-prefs}

%\todo{Change the indexing below to start with $t=1$.}

\begin{proof}
	Our proof is essentially identical to  the single-agent UCB analysis in Section 2.2 of~\citet{Bubeck12regret}.
	Assuming that the event \[F_{j, k}^{(i)}(t) = \left\{\ucb{i}{j}{t} >\ucb{i}{k}{t}\right\}\] is true, then at least 
	one of the three following events must occur: 
	\begin{equation*}
	\begin{split}
	\event_{1}(t) &= \left\{\hatmeanreward{i}{j}{t}  > \mu_j^{(i)} + \sqrt{\frac{3\log t}{2\bar{T}_j^{(i)}(t)}}\right\}, \\~
	\event_{2}(t) &= \left\{ 	\hatmeanreward{i}{k}{t} + \sqrt{\frac{3\log t}{2\bar{T}_k^{(i)}(t)}} < \mu^{(i)}_k\right\}, \\~
	\event_{3}(t) &= \left\{ \bar{T}_j^{(i)}(t) \leq \frac{6}{\Delta^2}\log (t)\right\}. 
	\end{split}
	\end{equation*}
	To see this, suppose that none of three events $\event_1(t)$, $\event_2(t)$ and $\event_3(t)$ occur. Then, 
	\begin{align*}
	\hatmeanreward{i}{k}{t} + \sqrt{ \frac{3\log (t)}{2\numpulls{i}{k}{t}}} &\geq \meanreward{i}{k} \ge  \meanreward{i}{j} + \mingap \ge  \meanreward{i}{j} + \sqrt{ \frac{6\log (t)}{\numpulls{i}{j}{t}}} \\
	&\geq  \hatmeanreward{i}{j}{t} + \sqrt{ \frac{3\log (t)}{2\numpulls{i}{j}{t}}}
	\end{align*}
	which is a contradiction because the left-hand side equals $\ucb{i}{k}{t}$ and the right-hand side equals $\ucb{i}{j}{t}$. 
	
	Let $u > 0$ be some value to be chosen later. Then, we have 
	\begin{align*}
	\sum_{t = 1}^{T } \Ind{F_{j, k}^{(i)}(t)\cap  \pull{i}(t) = j} &= \sum_{t = 1}^{T } \Ind{F_{j, k}^{(i)}(t)\cap  \pull{i}(t) = j \cap \numpulls{i}{j}{t} \leq u} \\
	&+ \sum_{t = 1}^{T } \Ind{F_{j, k}^{(i)}(t)\cap  \pull{i}(t) = j \cap \numpulls{i}{j}{t} > u}.
	\end{align*}
	Therefore, if we choose $u = \frac{6}{\Delta^2}\log (t)$,  we obtain 
	\begin{align*}
	\sum_{t = 1}^{T } &\Ind{F_{j, k}^{(i)}(t)\cap  \pull{i}(t) = j} = u + \sum_{t =  u + 1}^{T } \Ind{F_{j, k}^{(i)}(t)\cap  \pull{i}(t) = j \cap \numpulls{i}{j}{t} > u}\\
	&\leq u + \sum_{t =  \lfloor u \rfloor + 1}^{T } \Ind{\event_{1}(t)} + \sum_{t = \lfloor u \rfloor + 1}^{T } \Ind{\event_{2}(t)}.
	\end{align*}
	We are left to establish an upper bound on $\Pr(\event_{1}(t))$ and $\Pr(\event_{2}(t))$. We can do this by a simple application of a union bound and concentration:
	\begin{align*}
	\Pr(\event_{1}(t)) &\leq \Pr\left(\exists s \in \{1,2, \ldots t\} \colon \hatmeanreward{i}{j}{s} + \sqrt{ \frac{3\log (t)}{2s} }\leq \mu_j^{(i)} \right)\\
	&\leq \sum_{s = 1}^t  \Pr\left(\hatmeanreward{i}{j}{s} + \sqrt{ \frac{3\log (t)}{2s}} \leq \mu_j^{(i)} \right)\\
	&\leq \sum_{s = 1}^t \frac{1}{t^3} = \frac{1}{t^2},
	\end{align*}
	where the last inequality follows by a standard concentration argument for independent sub-Gaussian random variables. 
	The probability of $\event_{2}(t)$ occurring can be upper bounded similarly. Then, using $\sum_{t = 1}^{\infty} t^2 = \frac{\pi^2}{6}$ yields the conclusion. 
\end{proof}

\section{Proof of Proposition~\ref{prop:incentive-global}}\label{app:global-IC-proof}

\begin{proof}%[Proof of Proposition~\ref{prop:incentive-global}]
	By definition, player $\player{k}$'s regret can be lower-bounded as follows:
	\begin{equation}
	\label{eqn:IC-starting-point}
	\regret{k}(\horizon) \ge \sum_{i:  \arm{i}\psucc{k}\arm{k}} \gap{k}{i}\cdot \E\left[\numpulls{k}{i}{\horizon}\right] \ge \left(\min_{i: \gap{k}{i} < 0}\gap{k}{i}\right)\cdot \sum_{i:  \arm{i}\psucc{k}\arm{k}} \E\left[\numpulls{k}{i}{\horizon}\right].
	\end{equation}
	Since $\arm{i}\psucc{k}\arm{k}$ implies that $i < k$, we may proceed to upper bound $\sum_{i: i < k}\E\left[\numpulls{k}{i}{\horizon}\right]$. We claim that the following inclusion is true:
	\begin{align}\label{eq:superopt-pulls}
	\bigcup_{i = 1}^{k - 1}\left\{\pull{k}(t) = \arm{i}\right\} \subseteq \bigcup_{\substack{1\leq j < k \\ j < l \le \numarms}} \bigcup_{t-k \le t^\prime \le t}  \errorpot{j}{l}{t^\prime}.
	\end{align}
	The argument is as follows. If $\left\{\pull{k}(t) = \arm{i}\right\}$ holds for some $i$, we know that $\player{i}$ at time $t $ did not attempt to pull~$\arm{i}$. They either attempted to pull an arm $\arm{i^\prime}$ with $i^\prime > i$ or with $i^\prime < i$. Since we know that $\player{i}$ is running Algorithm~\ref{alg:ucb_ca_random}, in the former case, we can apply Lemma~\ref{lem:global_subopt} to player $\player{i}$; in the latter case, we can apply equation~\eqref{eq:superopt-attempts}, also to player $\player{i}$. This establishes equation~\eqref{eq:superopt-pulls}.
	
	Since all players $\player{j}$ with $j<k$ are running Algorithm~\ref{alg:ucb_ca_random}, we may apply Lemma~\ref{lem:ucb_bound} to yield
	\begin{equation}
	\sum_{i: i < k}\E\left[\numpulls{k}{i}{\horizon}\right] \leq 6 (k+1) \sum_{\substack{1\leq j <k \\ j < l \le \numarms}} \bigg(\frac{\log \horizon}{|\gap{j}{l}|^2} + 1\bigg)
	\le 6 k^2 (\numarms-k) \left( \frac{\log \horizon}{\Delta^2} + 1\right). 
	\end{equation}
	Substituting the above into \eqref{eqn:IC-starting-point} yields the desired lower bound.
\end{proof}

\section{Proof of Proposition~\ref{prop:not-incentive-comp}}\label{app:general-ic}

\begin{example}\label{ex:incentive-counterex}
	Let the set of players be $\players = \{ \player{1}, \player{2}, \player{3} \}$ and the set of arms be 
	$\arms = \{ \arm{1}, \arm{2}, \arm{3}\}$, with true preferences given by:
	\begin{align*}
	&\player{1}: \arm{1} \succ \arm{3} \succ \arm{2} &\quad \arm{1}: \player{2} \succ \player{1} \succ \player{3} \\
	&\player{2}: \arm{2} \succ \arm{1}  \succ \arm{3} &\quad \arm{2}: \player{3} \succ \player{2} \succ \player{1} \\
	&\player{3}: \arm{1} \succ \arm{3} \succ \arm{2} &\quad \arm{3}: \player{3} \succ \player{1} \succ \player{2}.
	\end{align*}
	The unique stable matching in this case is $(\player{1}, \arm{1}), (\player{2}, \arm{2}), (\player{3}, \arm{3})$.
\end{example}

\begin{proof}%[Proof of Proposition~\ref{prop:not-incentive-comp}]
	Let $\delaydraw{3}(t) \stackrel{i.i.d.}{\sim} Ber(\stayprob)$ for any $t$. The set of actions that player $\player{3}$ can play, for $t = 1,\cdots,\horizon$, to get negative stable regret is as follows: 
	\begin{equation}
	\attempt{3}(t)  = \begin{cases}
	\arm{2} &\text{ if } t= 3m-2 \\
	\arm{3} &\text{ if } t= 3m-1 \text{ and } \delaydraw{3}(t) = 0 \\
	\arm{2} &\text{ if } t= 3m-1 \text{ and } \delaydraw{3}(t) = 1 \\
	\arm{1} &\text{ if } t= 3m\\
	\end{cases}, \text{ for } m \in \mathbb{N}.
	\end{equation}

	By the definition of $\player{3}$'s regret, and using the fact that $\gap{k}{\emptyset}>\max\{\gap{3}{1} ,\gap{3}{2}\}>0$, we have:
	\begin{equation}
	\regret{3}(\horizon)\le \gap{3}{1}\cdot \E[\numpulls{3}{1}{\horizon}] + \gap{3}{\emptyset}\cdot \left(T- \E[\numpulls{3}{1}{\horizon}] \right).
	\end{equation}
	Thus it suffices to lower bound the expected number of times that $\player{3}$ successfully attempts $\arm{1}$.
	
	Define the following events:
	\begin{align*}
	\Omega^1_t &:= \{ F_{3, 1}^{(2)}(t) \cap \attempt{2}(t) = \arm{3}\}^c, \\	\Omega^2_t &:= \{ F_{1, 2}^{(2)}(t) \cap \attempt{2}(t) = \arm{1}\}^c. \\
	\end{align*}
	
	%Recall the definition of $E_t$ in the context of the current proof:  \[E_t \defeq \bigcap_{\player{i} \in \{\player{1}, \player{2}\}} \bigg\{\argmax_{\arm{j}\in \potset{i}(t)}\meanreward{i}{j} = \argmax_{\arm{j} \in \potset{i}(t)} \ucb{i}{j}{t}\bigg\}.\] Note that the intersection in the above display is only over the first two players, who are assumed to be running Algorithm~\ref{alg:ucb_ca_random}, and we ignore $\player{3}$, who is deviating.
	
	We first show the following inclusion, for any $m\in \mathbb{N}$:
	\begin{equation}\label{eq:incentive-inclusion}
\Omega^1_{3m-1}  \cap \Omega^2_{3m} \cap \{\delaydraw{2}(3m-1)=\delaydraw{3}(3m-1)=\delaydraw{2}(3m)=\delaydraw{1}(3m)=0 \}  \subseteq \{\pull{3}(3m) = \arm{1}\} . 
	\end{equation}
	We can simply check that this holds:
	\begin{itemize}
		\item At time $3m-2$, $\player{3}$ attempts and successfully pulls $\arm{2}$. 
		\item 	At time $3m-1$, $\player{2}$ pulls $\arm{1}$, since $\arm{2}$ is not in its plausible set, $\delaydraw{2}(3m-1)=0$ and the event~$\Omega^1_{3m-1}$ holds. $\player{3}$ pulls $\arm{3}$, since $\delaydraw{3}(3m-1)=0$.
		\item 	At time $3m$, $\player{1}$ does not pull $\arm{1}$, since $\arm{1}$ is not in its plausible set and $\delaydraw{1}(3m)=0$. $\player{2}$ does not pull $\arm{1}$, since $\arm{2}$ is in its plausible set, $\delaydraw{2}(3m) = 0$ and the event~$\Omega^2_{3m}$ holds. Thus $\player{3}$ successfully pulls $\arm{1}$.
	\end{itemize}
	Taking expectation of \eqref{eq:incentive-inclusion} and rearranging gives
	\begin{align}
	&\quad ~\Pr\left( \pull{3}(3m) = \arm{1} \right) \nonumber \\
	&\ge 1- \Pr \left(\left(\Omega^1_{3m-1}  \cap \Omega^2_{3m}  \cap \{\delaydraw{2}(3m-1)=\delaydraw{3}(3m-1)=\delaydraw{2}(3m)=\delaydraw{1}(3m)=0 \} \right)^c \right) \nonumber\\
	&\ge 1-\left( \Pr((\Omega^1_{3m-1})^c)+\Pr((\Omega^2_{3m})^c) + 4\stayprob\right), \label{eq:ic-union-bnd}
	\end{align}
	where the last inequality follows from a union bound. 
	
	It is useful to upper bound the following:
	\begin{align}
	&\sum_{m=1}^{\lfloor \horizon/3\rfloor} \Pr((\Omega^1_{3m-1})^c)+\Pr((\Omega^2_{3m})^c) \nonumber\\
	=&~
		\sum_{m=1}^{\lfloor \horizon/3\rfloor} 
		\Pr\left(F_{3, 1}^{(2)}(3m-1) \cap \attempt{2}(3m-1) = \arm{3}\right) +  \Pr\left( F_{1, 2}^{(2)}(3m) \cap \attempt{2}(3m) = \arm{1}\right) \nonumber\\
		\le &~\frac{1}{\stayprob(1-\stayprob)}\sum_{m=1}^{\lfloor \horizon/3\rfloor} 
	\Pr\left(F_{3, 1}^{(2)}(3m-1) \cap \pull{2}(3m-1) = \arm{3}\right) + 
	\Pr\left( F_{1, 2}^{(2)}(3m) \cap \pull{2}(3m) = \arm{1}\right) \nonumber\\
			\le &~\frac{1}{\stayprob(1-\stayprob)}\sum_{t=1}^{ \horizon} 
	\Pr\left(F_{3, 1}^{(2)}(t) \cap \pull{2}(t) = \arm{3}\right) + 
	\Pr\left( F_{1, 2}^{(2)}(t) \cap \pull{2}(t) = \arm{1}\right) \nonumber\\
	\le &~	\frac{1}{\stayprob(1-\stayprob)}\cdot 12 \cdot \left(\frac{1}{\Delta^2}\log (T) + 1\right), \label{eq:ic-ucb}
	\end{align} 
	where the last inequality follows from Lemma~\ref{lem:ucb_bound}.

	Now, we sum \eqref{eq:ic-union-bnd} over $m = 1, \cdots, \lfloor \horizon/3\rfloor$ to get:
	\begin{align*}
	\E[\numpulls{3}{1}{\horizon}] &\ge \sum_{m=1}^{\lfloor \horizon/3\rfloor} \Pr\left( \pull{3}(3m) = \arm{1} \right) \\
	&\ge (1-4\stayprob)\cdot \lfloor \horizon/3\rfloor - \sum_{m=1}^{\lfloor \horizon/3\rfloor}\Pr((\Omega^1_{3m-1})^c)+\Pr((\Omega^2_{3m})^c) \\
	& \ge (1-4\stayprob)\cdot \lfloor \horizon/3\rfloor -\frac{1}{\stayprob(1-\stayprob)}\cdot 12 \cdot \left(\frac{1}{\Delta^2}\log (T) + 1\right),
	\end{align*}
	where the last two inequalities follow from Equation~\eqref{eq:ic-union-bnd} and Equation~\eqref{eq:ic-ucb}.
	
	Thus we have
	\begin{equation}\label{eq:incent-final-bound}
	\regret{3}(\horizon)\le \gap{3}{1}\cdot \left((1-4\stayprob)\cdot \lfloor \horizon/3\rfloor - \frac{1}{\stayprob(1-\stayprob)}\cdot 12 \cdot \left(\frac{1}{\Delta^2}\log (\horizon) + 1\right)\right) + \gap{3}{\emptyset}\cdot \left(\frac{2}{3}\horizon\right).
	\end{equation}
	Note that $\gap{3}{1} < 0$, and $1-4\stayprob>0$ by assumption. Upon rearranging terms, we get the desired result.
	
\end{proof}

\newpage

\end{document}